\let\latexvec\vec
\documentclass[10pt, a4paper,onecolumn]{article}
\let\vec\latexvec 
\pdfoutput=1

\usepackage{appendix}
\usepackage[utf8]{inputenc} 
\usepackage[T1]{fontenc}    
\usepackage{lmodern}
\usepackage[hidelinks]{hyperref}       
\usepackage{bookmark}
\usepackage{url}            
\usepackage{booktabs}       
\usepackage{amsfonts}       
\usepackage{nicefrac}       
\usepackage{microtype}      
\usepackage{authblk}        
\usepackage[round]{natbib}
\renewcommand{\cite}{\citep}
\usepackage{import} 
\usepackage{mathtools} 
\usepackage{enumerate}
\usepackage{amsmath}
\usepackage{amsthm}
\usepackage{amssymb}

\usepackage{multirow} 
\usepackage{graphicx}

\usepackage{subfigure}
\usepackage{caption}
\usepackage{color}
\usepackage[dvipsnames]{xcolor}
\usepackage{algorithm}
\usepackage[compatible]{algpseudocode}
\ifpdf
  \DeclareGraphicsExtensions{.eps,.pdf,.png,.jpg}
\else
  \DeclareGraphicsExtensions{.eps}
\fi

\newcommand{\E}{\mathbb{E} \,} 
\newcommand{\R}{\mathbb{R} \,} 
\newcommand{\grad}{\nabla} 
\newcommand{\SUBSEQ}{\mathcal{S}}
\newcommand{\BUFFSUBSEQ}{{\mathcal{S}^*}}
\newcommand{\FULLSEQ}{{1:T}}

\makeatletter
\DeclareRobustCommand{\cev}[1]{%
  {\mathpalette\do@cev{#1}}%
}
\newcommand{\do@cev}[2]{%
  \vbox{\offinterlineskip
    \sbox\z@{$\m@th#1 x$}%
    \ialign{##\cr
      \hidewidth\reflectbox{$\m@th#1\vec{}\mkern4mu$}\hidewidth\cr
      \noalign{\kern-\ht\z@}
      $\m@th#1#2$\cr
    }%
  }%
}
\makeatother

\newcommand{\backwardmap}{\cev\psi}
\newcommand{\forwardmap}{\vec\psi}
\newcommand{\backwardkernel}{\cev\Psi}
\newcommand{\forwardkernel}{\vec\Psi}

\DeclareMathOperator*{\PF}{PF} 
\DeclareMathOperator*{\KSD}{KSD} 
\DeclareMathOperator*{\Var}{Var} 
\DeclareMathOperator*{\CoV}{CoV} 
\DeclareMathOperator*{\Corr}{Corr} 
\DeclareMathOperator*{\tr}{tr} 
\DeclareMathOperator{\logit}{logit} 

\newtheorem{theorem}{Theorem}
\newtheorem{lemma}{Lemma}
\newtheorem{assumption}{Assumption}

\begin{document}

\title{Stochastic Gradient MCMC for Nonlinear State Space Models}

\author[1]{Christopher Aicher}
\author[2]{Srshti Putcha}
\author[3]{Christopher Nemeth}
\author[3]{Paul Fearnhead}
\author[4]{Emily B. Fox}
\affil[1]{Department of Statistics, University of Washington}
\affil[2]{STOR-i Centre for Doctoral Training, Lancaster University}
\affil[3]{Department of Mathematics and Statistics, Lancaster University}
\affil[4]{Departments of Statistics and Computer Science, Stanford University}

\date{}
\maketitle

\begin{abstract} \noindent State space models (SSMs) provide a flexible framework for modeling complex time series via a latent stochastic process.
Inference for nonlinear, non-Gaussian SSMs is often tackled with particle methods that do not scale well to long time series. The challenge is two-fold: not only do computations scale linearly with time, as in the linear case, but particle filters additionally suffer from increasing particle degeneracy with longer series. Stochastic gradient MCMC methods have been developed to scale Bayesian inference for finite-state hidden Markov models and linear SSMs using buffered stochastic gradient estimates to account for temporal dependencies. We extend these stochastic gradient estimators to nonlinear SSMs using particle methods. We present error bounds that account for both buffering error and particle error in the case of nonlinear SSMs that are log-concave in the latent process. We evaluate our proposed particle buffered stochastic gradient using stochastic gradient MCMC for inference on both long sequential synthetic and minute-resolution financial returns data, demonstrating the importance of this class of methods.
\end{abstract}
\noindent 


\section{Introduction}
\label{sec:intro}
Nonlinear \emph{state space models} (SSMs) are widely used in many scientific domains for modeling time series.
For example, nonlinear SSMs can be applied in engineering (e.g. target tracking,~\citealt{gordon1993}), in epidemiology (e.g. compartmental disease models,~\citealt{dukic2012}),
and to financial time series (e.g. stochastic volatility models,~\citealt{shephard2005}).
To capture complex dynamical structure,
nonlinear SSMs augment the observed time series with a latent state sequence,
inducing a Markov chain dependence structure.
Parameter inference for nonlinear SSMs requires us to handle this latent state sequence. This is typically achieved using \emph{particle filtering} methods.

Particle filtering algorithms are a set of flexible Monte Carlo simulation-based methods, which use a set of samples, also known as \emph{particles}, to approximate the posterior distribution over the latent states. Unfortunately, inference in nonlinear SSMs does not scale well to long sequences: (i) the cost of each update requires full passes through the data that scales linearly with the length of the sequence, and (ii) the number of particles (and hence the computation per data point) required to control the bias of the particle filter scales linearly with the length of the sequence~\cite{kantas2015particle}.

Stochastic gradient Markov chain Monte Carlo (SG-MCMC) is a popular method for scaling Bayesian inference to large data sets, replacing full data gradients with stochastic gradient estimates based on subsets of data~\citep{welling2011bayesian, ma2015complete}.
In the context of SSMs, naive stochastic gradients are biased because subsampling breaks temporal dependencies in the data~\citep{ma2017stochastic, aicher2019stochastic}.
To correct for this, \citet{ma2017stochastic} and \citet{aicher2019stochastic} have developed \emph{buffered} stochastic gradient estimators that control the bias.
The latent state sequence is marginalized in a buffer around each subsequence, which reduces the effect that breaking dependencies has on the estimate of the gradient.
However, the work so far has been limited to SSMs where analytic marginalization is possible (e.g. finite-state HMMs and linear dynamical systems).

In this work, we propose \emph{particle buffered} gradient estimators that generalize the buffered gradient estimators to nonlinear SSMs.
Although straightforward in concept, a number of unique challenges arise in this setting.
First, we show how buffering in nonlinear SSMs can be approximated with a modified particle filter.
Second, we provide an error analysis of our proposed estimators by decomposing the error into subsequence error, buffering error, and particle filter error and analyze how this error propagates to estimating posterior means with SGMCMC.
Third, we extend the buffering error bounds of \citet{aicher2019stochastic} to nonlinear SSMs with log-concave likelihoods and show that buffer error decays geometrically in buffer size, ensuring that a small buffer size can be used in practice.

The theory we present highlights the importance of controlling bias in  the estimate of the gradient -- as whilst the impact of a high variance estimator on the accuracy of the SG-MCMC algorithm can be controlled by increasing the number of steps and reducing the step size, it is not possible to change the implementation of the SG-MCMC algorithm to reduce the impact of the bias. We then show theoretically that introducing buffering enables us to control the bias of the estimates of the gradient -- with the bias decaying geometrically in the size of the buffer. We investigate the accuracy of our new approach on a range of models with both synthetic and real data -- and show that for fixed computational cost we have obtained substantial gains in accuracy over alternatives. This is due to the reduced bias relative to unbuffered versions of SG-MCMC and through the fact that using stochastic gradient methods allows for more iterations of the MCMC algorithm when compared to approaches that estimate gradients using all observations.

Python code for our Algorithm and for replicating our numerical studies is available at \url{https://github.com/aicherc/sgmcmc_ssm_code}.


\section{Background}
\label{sec:background}

\subsection{Nonlinear State Space Models for Time Series}
\label{sub:statespace}
State space models are a class of discrete-time bivariate
stochastic processes  consisting of a latent state process $X = \{X_t \in
\mathbb{R}^{d_x} \}_{t = 1}^T$ and a second observed process, $Y =\{Y_t \in \mathbb{R}^{d_y} \}_{t = 1}^T$.
The evolution of the state variables is typically assumed to be a time-homogeneous Markov process, such that the latent state at time $t$, $X_{t}$, is determined only by the latent state at time $t-1$, $X_{t-1}$.
The observed states are conditionally independent given the latent states.
Given the prior $X_0 \sim \nu(x_0 | \theta)$ and parameters $\theta \in \Theta$, the generative model for $X,Y$ is thus
\begin{align} \label{eq:ssm}
  X_t | (X_{t-1} = x_{t-1}, \theta) &\sim p(x_t \,|\, x_{t-1}, \theta),\\
  Y_t | (X_t = x_t, \theta) &\sim p(y_t \,|\, x_t, \theta), \nonumber
\end{align}
where we call $p(x_t \,|\, x_{t-1}, \theta)$ the \textit{transition density} and $p(y_t \,|\, x_{t}, \theta)$ the \textit{emission density}.

For an arbitrary sequence $\{z_i\}$, we use $z_{i:j}$ to denote the
sequence $(z_i, z_{i+1}, \ldots, z_j)$. 
To infer the model parameters $\theta$, a quantity of interest is the \textit{score function}, the gradient of the marginal loglikelihood, $\grad_\theta \log p(y_{1:T} | \theta)$.
Using the score function, the loglikelihood can  be maximized iteratively via a (batch) \textit{gradient ascent} algorithm \citep{robbins1951}, given the observations, $y_{1:T}$.

If the latent state posterior $p(x_{1:T} | y_{1:T}, \theta)$ can be expressed analytically,
we can calculate the score using \emph{Fisher's identity} \citep{cappe2005inference},
\begin{align}
\nonumber
\grad_\theta \log p(y_{1:T} \, | \, \theta) &= \E_{X | Y, \theta}[ \grad_\theta \log p(X_{1:T}, y_{1:T} \, | \, \theta)] \\
    =&  \sum_{t=1}^T \E_{X|Y, \theta}[\grad_\theta \log p(X_t, y_t \,| \, x_{t-1}, \theta)].
 \label{eq:fisher_identity}
\end{align}
If the latent state posterior, $p(x_{1:T} | y_{1:T}, \theta)$, is not available in closed-form, we can approximate the expectations of the latent state posterior.
One popular approach is via \textit{particle filtering} methods.

\subsubsection{Particle Filtering and Smoothing}
\label{subsub:pf}
\textit{Particle filtering} algorithms \citep[see e.g.][]{doucet2009tutorial, fearnhead2018} can be used to create an empirical approximation of the expectation of a function $H(X_{1:T})$
with respect to the posterior density, $p(x_{1:T} | y_{1:T}, \theta)$.
This is done by generating a collection of $N$ random samples or \textit{particles}, $\{x_t^{(i)}\}_{i=1}^N$ and calculating their associated importance weights, $\{w_t^{(i)}\}_{i=1}^N$, recursively over time.
We update the particles and weights with \textit{sequential importance resampling} \citep{doucet2009tutorial} in the following manner.

\begin{enumerate}[(i)]
  \item \textit{Resample} auxiliary ancestor indices $\{a_1, \ldots, a_N\}$ with probabilities proportional to the importance weights, i.e. $a_i \sim \text{Categorical}(w_{t-1}^{(i)})$.
  \item \textit{Propagate} particles $x_t^{(i)} \sim q(\cdot | x_{t-1}^{(a_i)}, y_t, \theta)$, using a proposal distribution $q(\cdot|\cdot)$.

  \item \textit{Update} and normalize the weight of each particle,
\begin{align} \label{eq:reweight}
      w_t^{(i)} \propto \frac{p(y_t | x_t^{(i)}, \theta) p(x_t^{(i)} | x_{t-1}^{(a_i)}, \theta)}{q(x_t^{(i)} | x_{t-1}^{(a_i)}, y_t, \theta)}  \enspace, \enspace \sum_i w_t^{(i)} = 1\enspace.
\end{align}
\end{enumerate}

The auxiliary variables, $\{a_i\}_{i=1}^N$, represent the indices of the \emph{ancestors} of the particles, $\{x_t^{(i)}\}_{i=1}^N$, sampled at time $t$. The introduction of ancestor indices allows us to keep track of the lineage of particles over time \citep{andrieu2010}. The \emph{multinomial resampling} scheme given in (i) describes the procedure by which \textit{offspring} particles are produced.

Resampling at each iteration is used to mitigate against the problem of \emph{weight degeneracy}. This phenomenon occurs when the variance of the importance weights grows, causing more and more particles to have negligible weight. Aside from the multinomial resampling scheme described above, there are various other resampling schemes outlined in the particle filtering literature, such as stratified sampling \citep{kitagawa1996} and residual sampling \citep{liu1998}.

If the proposal density $q( x_t | x_{t-1}, y_t, \theta)$ is the transition density $p( x_t | x_{t-1}, \theta)$ we obtain the \emph{bootstrap particle filter} \citep{gordon1993}. By using the transition density for proposals, the importance weight recursion in \eqref{eq:reweight} simplifies to $w_t^{(i)} \propto p(y_t | x_t^{(i)}, \theta)$.

When our target function decomposes into a pairwise sum $H(x_{1:T}) = \sum_{t = 1}^T h_t(x_t, x_{t-1})$ -- such as for Fisher's identity $h_t(x_t, x_{t-1}) = \grad_\theta \log p(y_t, x_t \, | \, x_{t-1}, \theta)$ -- then we only need to keep track of the partial sum $H_t = \sum_{s = 1}^t h_s(x_s, x_{s-1})$ in the filter~\cite{doucet2009tutorial}: see Algorithm~\ref{alg:pf}. 

\begin{algorithm}[h]
   \caption{Particle Filter}
   \label{alg:pf}
\begin{algorithmic}[1]
  \STATE {\bfseries Input:} number of particles, $N$, pairwise statistics, $h_{1:T}$, observations $y_{1:T}$, proposal density $q$,
  \STATE Draw $x_0^{(i)} \sim  \nu(x_0 | \theta)$, set $w_0^{(i)} = \frac{1}{N}$, and $H_0^{(i)} = 0$ $\forall i$.
  \FOR{$t = 1, \ldots, T$}
  \STATE Resample ancestor indices $\{a_1, \ldots, a_N\}$.
  \STATE Propagate particles $x_t^{(i)} \sim q(\cdot | x_{t-1}^{(a_i)}, y_t, \theta)$.
  \STATE Update each $w_t^{(i)}$ according to \eqref{eq:reweight}.
  \STATE Update statistics $H_t^{(i)} = H_{t-1}^{(a_i)} + h_t(x_t^{(i)}, x_{t-1}^{(a_i)})$.
  \ENDFOR
  \STATE Return $H = \sum_{i = 1}^N w_T^{(i)} H_T^{(i)}$.
\end{algorithmic}
\end{algorithm}
A key challenge for particle filters is handling large $T$.
Not only do long sequences require $\mathcal{O}(T)$ computation, but particle filters require a large number of particles, $N$, to avoid \emph{particle degeneracy}: the use of resampling in the particle filter causes path-dependence over time, depleting the number of distinct particles available overall.
For Algorithm~\ref{alg:pf}, the variance in $H$ scales as $\mathcal{O}(T^2/N)$~\citep{poyiadjis2011particle}.
Therefore to maintain a constant variance, the number of particles would need to increase quadratically with $T$, which is computationally infeasible for long sequences.
 \citet{poyiadjis2011particle, nemeth2016particle} and \citet{olsson2017efficient} propose alternatives to Step 7 of Algorithm~\ref{alg:pf} that trade additional computation or bias to decrease the variance in $H$ to $\mathcal{O}(T/N)$.
Fixed-lag particle smoothers provide another approach to avoid particle degeneracy, where sample paths are not updated after a fixed lag~\citep{kitagawa2001monte, dahlin2015particle}.
All of these methods perform a full pass over the data $y_{1:T}$, which requires $\mathcal{O}(T)$ computation.

\subsection{Stochastic Gradient MCMC}
\label{sub:sgmcmc}

One popular method to conduct scalable Bayesian inference for large data sets is \emph{stochastic gradient} Markov chain Monte Carlo (SGMCMC). 
Given a prior $p(\theta)$, to draw a sample $\theta$ from the posterior $p(\theta | y) \propto p(y | \theta) p(\theta)$,
gradient-based MCMC methods simulate a stochastic differential equation (SDE) based on the gradient of the loglikelihood $g_\theta = \grad_\theta \log p(y | \theta)$, such that the posterior is the stationary distribution of the SDE.
SGMCMC methods replace the full-data gradients with stochastic gradients, $\widehat{g}_\theta$, using subsamples of the data to avoid costly computation.

The most common method of the SGMCMC family is the \textit{stochastic gradient Langevin dynamics} (SGLD) algorithm~\citep{welling2011bayesian, nemeth2021stochastic}:
\begin{equation}
\label{eq:sgld}
\theta^{(k+1)} \leftarrow \theta^{(k)} + \epsilon^{(k)} \cdot (\widehat{g}_\theta + \grad \log p(\theta)) + \mathcal{N}(0, 2\epsilon^{(k)}),
\end{equation} 
where $\epsilon^{(k)}$ is the stepsize and $\theta_1$ is an initialization of the chain.
When $\widehat{g}_\theta$ is unbiased and with an appropriate decreasing stepsize, the distribution of $\theta^{(k)}$ asymptotically converges to the posterior distribution~\citep{teh2016consistency}.
\citet{dalalyan2019user} provide non-asymptotic bounds on the Wasserstein distance between the posterior and the output of SGLD after $K$ steps for fixed $\epsilon^{(k)} = \epsilon$ and possibly biased $\widehat{g}_\theta$.

Many extensions of SGLD exist in the literature, including using control variates to reduce the variance of $\widehat{g}_\theta$~\citep{baker2019control, nagapetyan2017true, chatterji2018theory} and augmented dynamics to improve mixing~\citep{ma2015complete} such as stochastic gradient Hamiltonian Monte Carlo~\citep{chen2014}, stochastic gradient Nos\'e-Hoover thermostat~\citep{ding2014bayesian}, and stochastic gradient Riemannian Langevin dynamics~\citep{girolami2011riemann, patterson2013stochastic}.

\subsubsection{Stochastic Gradients for SSMs}
An additional challenge when applying SGMCMC to SSMs is handling the temporal dependence between observations.
Based on a subset $\SUBSEQ$ of size $S$, an unbiased stochastic gradient estimate of \eqref{eq:fisher_identity} is
\begin{equation}
\label{eq:unbiased_sg}
\sum_{t \in \SUBSEQ} \Pr(t \in \SUBSEQ)^{-1} \cdot \E_{X | y_{1:T}, \theta}[ \grad_\theta \log p(X_t, y_t \, | \, X_{t-1}, \theta)].
\end{equation}
Although \eqref{eq:unbiased_sg} is a sum over $S$ terms, it requires taking expectations with respect to $p(x | y_{1:T}, \theta)$, which requires processing the full sequence $y_{1:T}$. 
One approach to reduce computation is to randomly sample $\SUBSEQ$ as a contiguous subsequence $\SUBSEQ = \{s+1, \ldots, s+S\}$ and approximate \eqref{eq:unbiased_sg} using only $y_\SUBSEQ$
\begin{equation}
\label{eq:naive_sg}
\sum_{t \in \SUBSEQ} \Pr(t \in \SUBSEQ)^{-1} \cdot \E_{X | y_\SUBSEQ, \theta}[ \grad_\theta \log p(X_t, y_t \, | \, X_{t-1}, \theta)].
\end{equation}
However, \eqref{eq:naive_sg} is \emph{biased} because the expectation over the latent states $x_\SUBSEQ$ is conditioned only on $y_\SUBSEQ$ rather than $y_{1:T}$.

To control the bias in stochastic gradients while also avoiding accessing the full sequence,
previous work on SGMCMC for SSMs proposed \emph{buffered} stochastic gradients~\citep{ma2017stochastic, aicher2019stochastic}.
\begin{equation}
\widehat{g}_\theta(S, B) = \sum_{t \in \SUBSEQ} \frac{\E_{X | y_{\mathcal{S^*}}, \theta} [\grad_\theta \log p(X_t, y_t  \, | \, X_{t-1}, \theta)]}{\Pr(t \in \SUBSEQ)},
\label{eq:buffered_sg}
\end{equation}
where $\BUFFSUBSEQ = \{s+1-B, \ldots, s+S+B\}$ is the \emph{buffered} subsequence such that $\SUBSEQ \subseteq \BUFFSUBSEQ \subseteq \{1, \ldots, T\}$ (see Figure~\ref{fig:ssm_buffer}).
When the "buffer" extends outside of the original subsequence (e.g. $s+1-B < 1$ or $s+S+B > T$), then we can extend the model to $\{1-B, \ldots, T+B\}$ and assume the observations $y_t$ outside of $\{1, \ldots, T\}$ are missing.
In practice, we will truncate $\BUFFSUBSEQ$ by intersecting it with $\{1, \ldots, T\}$.

The unbiased gradient estimate, which conditions on all data \eqref{eq:unbiased_sg}, is $\widehat{g}(S, T)$ and the estimator with no buffering \eqref{eq:naive_sg} is $\widehat{g}(S, 0)$.
As $B$ increases from $0$ to $T$, the estimator $\widehat{g}_\theta(S, B)$ trades computation for reduced bias.
\label{sec:framework-approx-gradient}
\begin{figure}[ht]
\centering
\begin{minipage}[c]{.48\textwidth}
    \centering
    \includegraphics[width=0.99\linewidth, trim=0.2in 0in 0.2in 0in]{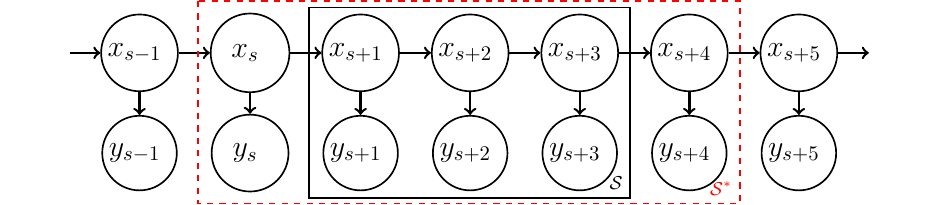}
\end{minipage}
\caption{Graphical model of $\BUFFSUBSEQ$ with $S = 3$ and $B = 1$.}
\label{fig:ssm_buffer}
\end{figure}
In particular, when the model and gradient both satisfy a Lipschitz property, the error decays geometrically in buffer size $B$, see Theorem 4.1 of~\citet{aicher2019stochastic}.
Specifically, for all $\SUBSEQ$
\begin{equation}
\label{eq:geometric_bound}
\| \widehat{g}_\theta(S, B) - \widehat{g}_\theta(S,T) \|_2 = \mathcal{O}(L_\theta^{B} \cdot T / S),
\end{equation}
where $L_\theta$ is a bound for the Lipschitz constants of the \emph{forward and backward smoothing kernels}\footnote{We follow~\citet{aicher2019stochastic} and consider Lipschitz constants for a kernel $\Psi$ measured in terms of the $p$-Wasserstein distance between distributions of $x, x'$ and $\Psi(x), \Psi(x')$.}
\begin{align} \label{eq:smooth_kernel}
\nonumber 
\forwardkernel_t(x_{t+1}, x_t) = p(x_{t+1} \, | \, x_t, y_{1:T}, \theta), \\
\backwardkernel_t(x_{t-1}, x_t) = p(x_{t-1} \, | \, x_t, y_{1:T}, \theta).
\end{align}

The bound provided in \eqref{eq:geometric_bound} ensures that only a modest buffer size $B$ is required (e.g. $\mathcal{O}(\log \delta^{-1})$ for an accuracy of $\delta$). Unfortunately, neither the buffered stochastic gradient $\widehat{g}_\theta(S,B)$ nor the smoothing kernels $\{\forwardkernel_t,\backwardkernel_t\}$ have a closed form for nonlinear SSMs.


\section{Method}
\label{sec:method}
In this section, we propose a particle buffered stochastic gradient for nonlinear SSMs, by applying the particle approximations of Section~\ref{sub:statespace} to \eqref{eq:buffered_sg}.

\subsection{Buffered Stochastic Gradient Estimates for Nonlinear SSMs}
Let $g_\theta^{\PF}(S,B,N)$ denote the particle approximation of $\widehat{g}_\theta(S,B)$ with $N$ particles.
We approximate the expectation over $p(x | y_\BUFFSUBSEQ, \theta)$ in \eqref{eq:buffered_sg} using Algorithm~\ref{alg:pf} run over $\BUFFSUBSEQ$. In the following we will use $\nu_0$ as the prior distribution for $X_{s+1-B}$, which is a natural choice if the state process is stationary and $\nu_0$ is its stationary distribution; for other cases better choices for the prior distribution of  $X_{s+1-B}$ may be possible.

The complete data loglikelihood, $\log p(y_\SUBSEQ, x_\SUBSEQ, \theta)$, in \eqref{eq:buffered_sg} decomposes into a sum of pairwise statistics 
\begin{equation}
H = \sum_{t \in \BUFFSUBSEQ} h_t(x_t, x_{t-1}) \enspace,
\end{equation}
where
\begin{equation}
h_t(x_t, x_{t-1}) = \begin{cases}
    \dfrac{\grad_\theta \log p(x_t, y_t \, | \, x_{t-1}, \theta)}{\Pr(t \in \SUBSEQ)} & \text{ if } t \in \SUBSEQ, \\
    0 & \text{ otherwise.}
    \end{cases}
\label{eq:complete_loglike_pairwise_statistics}
\end{equation}
We highlight that the statistic is zero for $t$ in the left and right buffers $\BUFFSUBSEQ \backslash \SUBSEQ$.
Although $H_t$ is not updated by $h_t$ for $t$ in $\BUFFSUBSEQ \backslash \SUBSEQ$,
running the particle filter over the buffers is \emph{crucial} to reduce the bias of $g_\theta^{\PF}(S,B,N)$.

Note that $g_\theta^{\PF}(S,B,N)$ allows us to approximate the non-analytic expectation in \eqref{eq:buffered_sg} with a modest number of particles $N$, by avoiding the particle degeneracy and full sequence runtime bottlenecks, as the particle filter is only run over $\BUFFSUBSEQ$, which has length $S+2B \ll T$.

\subsection{SGMCMC Algorithm}
\label{sub:sgld_alg}
Using $g_\theta^{\PF}(S,B,N)$ as our stochastic gradient estimate in SGLD, \eqref{eq:sgld}, gives us Algorithm~\ref{alg:sgld}.
\begin{algorithm}
\caption{Buffered PF-SGLD}
\label{alg:sgld}
\begin{algorithmic}[1]
\STATE{Input: data $y_{1:T}$, initial $\theta^{(0)}$, stepsize $\epsilon$, subsequence size $S$, buffer size $B$, particle size $N$}
\FOR{$k = 1, 2, \ldots, K$}
\STATE{Sample $\SUBSEQ = \{s+1, \ldots, s+S\}$}
\STATE{Set $\BUFFSUBSEQ = \{s+1-B, \ldots, s+S+B\}$.}
\STATE{Calculate $\, g_\theta^{\PF}$ over $\BUFFSUBSEQ$ using Alg.~\ref{alg:pf} on \eqref{eq:complete_loglike_pairwise_statistics}.}
\STATE{Set $\, \theta^{(k+1)} \leftarrow \theta^{(k)} + \epsilon \cdot (g_\theta^{\PF} + \grad \log p(\theta)) + \mathcal{N}(0, 2\epsilon)$}
\ENDFOR
\STATE{Return $\theta^{(K+1)}$}
\end{algorithmic}
\end{algorithm}

Algorithm~\ref{alg:sgld} can be extended by (i) averaging over multiple sequences or varying the subsequence sampling method~\citep{schmidt2015non,ou2018clustering}, (ii) using different particle filters such as those listed in Section~\ref{subsub:pf}, and (iii) using more advanced SGMCMC schemes such as those listed in Section~\ref{sub:sgmcmc}.


\section{Error Analysis}
\label{sub:error}
In this section, 
we analyze the error of our particle buffered stochastic gradient $g^{\PF}_\theta$ and its effect on approximating posterior means with finite sample averages using Algorithm~\ref{alg:sgld}.
We first present error bounds for approximating posterior means using SGLD with biased gradients (Theorem~\ref{thm:finite_sample_error}).
We then present bounds on the gradient bias and MSE of $g^{\PF}_\theta$, extending the error bounds of~\citet{aicher2019stochastic} (Theorem~\ref{thm:gradient_bias_mse}).
In particular, we provide bounds for the Lipschitz constant $L_\theta$ of the smoothing kernels \eqref{eq:smooth_kernel} without requiring an explicit form for the smoothing kernels (Theorem~\ref{thm:lipschitz_bound_logconcave}), allowing \eqref{eq:geometric_bound} to apply to nonlinear SSMs.

\subsection{Error of Biased SGLD's Finite Sample Averages}
\label{subsub:finite_sample_error}
We consider the estimation error of the posterior expected value of some test function of the parameters $\phi:\Theta \rightarrow \R$ using samples $\theta^{(k)}$ drawn using SGLD with a fixed step size $\epsilon$ and stochastic gradients $g_\theta$.

Let $\bar{\phi}$ be the posterior expected value
\begin{equation}
\bar{\phi} = \E_{p(\theta | y)}[\phi(\theta)] \enspace,
\end{equation}
and let $\hat{\phi}_{K,\epsilon}$ be the $K$-sample estimator for $\bar{\phi}$
\begin{equation}
\hat{\phi}_{K, \epsilon} = \frac{1}{K} \sum_{k = 1}^K \phi(\theta^{(k)}) \enspace.
\end{equation}

The error of the \emph{finite sample average} $|\hat{\phi}_{K, \epsilon} - \bar{\phi}|$ has been previously studied for SGLD with \emph{unbiased} gradients by~\citet{vollmer2016non} and \citet{chen2015convergence}.
Following \citet{chen2015convergence}, we make the following assumption on $\phi$.
\begin{assumption}
\label{asmp:psi}
Let $\mathcal{L}$ be the generator of the Langevin diffusion
\begin{equation*}
\mathcal{L}[\psi(\theta_t)] = -\nabla\log p(\theta_t) \cdot \nabla \psi(\theta_t) + \frac{\epsilon^2}{2} \tr(\nabla^2\psi(\theta_t)) 
\enspace.
\end{equation*}
Then, we define $\psi$ to solve the \emph{Poisson equation}
\begin{equation}
\label{eq:poisson}
\frac{1}{K} \sum_{k = 1}^K \mathcal{L}[\psi(\theta^{(k)})] = \hat\phi_{K,\epsilon} - \bar\phi \enspace.
\end{equation}
We assume that $\psi(\theta)$ and its derivatives (up to third order) are bounded.
\end{assumption}

We now present Theorem~\ref{thm:finite_sample_error}, which bounds the error of a finite sample Monte Carlo estimator based on SGLD when the stochastic gradients $\hat{g}_\theta$ are potentially \emph{biased}.

\begin{theorem}[Error of Finite Sample Average]
\label{thm:finite_sample_error}
If the gradient $g_\theta$ is smooth in $\theta$, the test function $\phi$ satisfies a moment condition (Assumption~\ref{asmp:psi}) and the bias and MSE of the gradient estimates $\hat{g}_\theta$ are uniformly bounded, that is,
\begin{equation}
\| \E \hat{g}_\theta - g_\theta \| \leq \delta \text{ and } \E \| \hat{g}_\theta - g_\theta \|^2 \leq \sigma^2 \text{ for all } \theta  \enspace,
\end{equation}
then there exists some constant $C >0$, such that the bias and MSE of $\hat{\phi}_{K, \epsilon}$ satisfy
\begin{align}
\label{eq:sgld_bias}
|\E \hat{\phi}_{K, \epsilon} - \bar{\phi} \, | &\leq C \cdot \left( \frac{1}{K \epsilon} + \delta \right) + \mathcal{O}(\epsilon)  \enspace,
\\
\label{eq:sgld_mse}
\E|\hat{\phi}_{K, \epsilon} - \bar{\phi}\,|^2 &\leq C \left( \frac{1}{K^2 \epsilon^2} + \frac{\sigma^2}{K} + \delta^2 + \frac{\delta}{\epsilon}\right) + \mathcal{O}\left(\frac{1}{K \epsilon} + \delta \epsilon + \epsilon^2\right) \enspace.
\end{align}
\end{theorem}
The bias bound, \eqref{eq:sgld_bias}, is a direct application of Theorem 2 in~\citet{chen2015convergence}.
The MSE bound, \eqref{eq:sgld_mse}, is an extension of Theorem 3 in~\citet{chen2015convergence} when the stochastic gradient estimates $\hat{g}_\theta$ are biased (i.e. $\delta \neq 0$).
The additional bias terms $\delta$ arise from keeping track of additional cross terms in $(\hat{\phi}_{K,\epsilon} - \bar{\phi})^2$.
The proof of Theorem~\ref{thm:finite_sample_error} is presented in the Appendix.

From Theorem~\ref{thm:finite_sample_error}, we see that the error bounds on $\hat{\phi}_{K, \epsilon}$ are more sensitive to the bias $\delta$ of $\hat{g}$ than the variance $\sigma^2$:
the term involving $\sigma^2$ decays with increasing $K$, while terms involving $\delta$ do not decay regardless of stepsize $\epsilon$ or number of samples $K$. A similar conclusion comes from the bound on error of SGLD in Theorem 4 of \citet{dalalyan2019user}: the impact of bias on the error bound is not affected by step size, whereas the impact of the variance can be reduced by taking more steps of smaller size; 
however, we do not require the posterior distribution be log-concave.

Therefore for the samples from Algorithm~\ref{alg:sgld} to be useful, it is important for the bias of $g^{\PF}_\theta$ to be controlled.


\subsection{Gradient Bias and MSE Bounds}
\label{subsub:gradient_error_bound}
To apply Theorem~\ref{thm:finite_sample_error} to the samples from Algorithm~\ref{alg:sgld},
we develop bounds on the bias $\delta$ and MSE $\sigma^2$ of our particle buffered stochastic gradients $g^{\PF}_\theta$.

\begin{theorem}[Bias and MSE Bounds for $g^{\PF}_\theta$]
\label{thm:gradient_bias_mse}
For fixed $\theta$, if the model and gradient satisfy a Lipschitz condition and
there is a bound on the autocorrelation between $\mathbb{E}_{X|y_{1:T}} \nabla \log p(y_t, X_t | X_{t-1}, \theta)$ for different $t$,
then the bias $\delta$ and MSE $\sigma^2$ of $g^{\PF}_\theta$ is bounded by
\begin{align}
\label{eq:grad_bias_bound}
\delta &\leq \gamma \cdot \left[C_1 \cdot L_\theta^{B} + \mathcal{O}\left(\frac{S + 2B}{N}\right)\right] \enspace,
\\
\label{eq:grad_mse_bound}
\sigma^2 &\leq 3 \gamma^2 \cdot \left[C_1^2 \cdot L_\theta^{2B} + C_2 S + \mathcal{O}\left(\frac{(S+2B)^2}{N}\right)\right] \enspace,
\end{align}
where $\gamma = \max_t \Pr(t \in \SUBSEQ)^{-1}$ and $C_1, C_2$ are constants with respect to $S,B,N$.
\end{theorem}
From Theorem~\ref{thm:gradient_bias_mse}, we see that the bias $\delta$ \eqref{eq:grad_bias_bound} can be controlled by selecting large enough $N$ and $B$ when $L_\theta < 1$.

We now sketch the proof of Theorem~\ref{thm:gradient_bias_mse} and discuss its assumptions. The complete proof can be found in the Appendix.

We decompose the error between $g^{\PF}_\theta$ and the full gradient $g_\theta$ through $\hat{g}_\theta(S,B)$ and $\hat{g}_\theta(S, T)$ into three error sources:
\begin{align}
\nonumber
\| g^{\PF}_\theta(S, &B, N) - g_\theta \| \leq
\underbrace{\| g^{\PF}_\theta(S,B,N) - \hat{g}_\theta(S,B) \|}_{\text{particle error (I)}} \ + \\
 \label{eq:error_decompose}
 &\underbrace{\| \hat{g}_\theta(S,B) - \hat{g}_\theta(S,T) \|}_{\text{buffering error (II)}}\ + 
 \underbrace{\| \hat{g}_\theta(S,T) - g_\theta \|}_{\text{subsequence error (III)}}.
\end{align}
\begin{enumerate}[(I)]
\item \emph{Particle error}:
the Monte Carlo error of the particle filter.
From~\citet{kantas2015particle}, the asymptotic bias and MSE of a particle approximation to the sum of $R$ test functions (using Algorithm~\ref{alg:pf}) is $\mathcal{O}(R/N)$ and $\mathcal{O}(R^2/N)$ respectively.
Since $g^{\PF}(S,B,N)$ is a particle approximation to the sum of $R = S+2B$ test functions (i.e., $h_t(x_t, x_{t-1})$), we have
\begin{align}
\nonumber
\| \E g^{\PF}_\theta(S,B,N) -\hat{g}_\theta(S,B) \|&= \mathcal{O}\left(\gamma \cdot \frac{S+2B}{N}\right) \\
\E \| g^{\PF}_\theta(S,B,N) -\hat{g}_\theta(S,B) \|^2 &= \mathcal{O}\left(\gamma^2 \cdot \frac{(S+2B)^2}{N}\right) \enspace,
\end{align}
where $\gamma$ is a upper bound on the sampling scale factor
$\gamma = \max_t \Pr( t \in \SUBSEQ)^{-1}$.

Using a more advanced particle filter, such as the ``PaRIS'' or ``Poyiadjis $N^2$'' algorithm, Corollary 6 of~\citet{olsson2017efficient} gives a tighter bound for the MSE
\begin{align*}
\E \| g^{\PF}_\theta(S,B,N) -\hat{g}_\theta(&\,S,B) \|^2 = \mathcal{O}\left(\gamma^2 \cdot \frac{S+2B}{N}\right) \enspace.
\end{align*}
However in our experiments, we found that the improved MSE of these other particle filters was not worth the additional computational overhead for the small subsequences we considered, where $S+2B \lesssim 100$. See experiments in the Appendix.

\item \emph{Buffering error},:
error in approximating the latent state posterior $p(x_\FULLSEQ  |  y_\FULLSEQ)$ with $p(x_\FULLSEQ  |  y_{\BUFFSUBSEQ})$. The error stems from conditioning on only a buffered subsequence $y_{\BUFFSUBSEQ}$ instead of $y_\FULLSEQ$ and the initial distribution approximation $\nu_0$ for $X_{s+1-B}$.
If the smoothing kernels $\{\forwardkernel_t, \backwardkernel_t\}$ are contractions for all $t$ (i.e. $L_\theta < 1$),
then according to \eqref{eq:geometric_bound}, the error in this term is proportional to $\gamma L_\theta^B$.
In Section~\ref{subsub:buff_error_analysis}, we show sufficient conditions for $L_\theta < 1$.

\item \emph{Subsequence error}: the error in approximating Fisher's identity using a randomly chosen subsequence of data points.
The error in this term depends on the subsequence size $S$ and how subsequences are sampled.
Because we sample random \emph{contiguous} subsequences of size $S$,
the MSE scales $\mathcal{O}(\gamma^2 S \tfrac{1+\rho}{1-\rho})$,
where $\rho$ is a bound on the autocorrelation between $\mathbb{E}_{X|y_{1:T}} \nabla \log p(y_t, X_t | X_{t-1}, \theta)$ for different $t$.
See the Appendix for details.
\end{enumerate}
Combining these error bounds gives us Theorem~\ref{thm:gradient_bias_mse}.

We present examples of the asymptotic bias and MSE bounds given by Theorem~\ref{thm:gradient_bias_mse} for four different gradient estimators in Table~\ref{tab:asymp_error_rates}.
The four gradient estimators are: 
(i) naive stochastic subsequence (without buffering) $g^{\PF}(S, 0, N)$
(ii) buffered stochastic subsequence $g^{\PF}(S,B,N)$,
(iii) fully buffered stochastic subsequence $g^{\PF}(S,T,N)$,
and
(iv) full sequence $g^{\PF}(T, T, N)$.
For simplicity, we assume the subsequences $\SUBSEQ$ are sampled from a \emph{strict} partition of $\FULLSEQ$ such that $\gamma = T/S$ and assume $B$ is on the same order as $S$ (i.e. $B$ is $\mathcal{O}(S)$).

\begin{table*}[htb]
\centering
\caption{Asymptotic bias and compute cost for four different gradient estimators.
}
\label{tab:asymp_error_rates}
\begin{tabular}{c  c  c  c  c }
Gradient & $(S,B,N)$ & Bias $\delta$ & Compute \\
\midrule
Naive Subsequence & $(S, 0, N)$ & $C_1 \cdot T/S + \mathcal{O}(T/N)$ & $\mathcal{O}(SN)$ \\
Buffered Subsequence & $(S, B, N)$ & $C_1 \cdot L_\theta^{B}  \cdot T/S + \mathcal{O}(T/N)$ & $\mathcal{O}(SN)$ \\
Fully Buffered Subsequence & $(S, T, N)$ & $\mathcal{O}(T/N)$ & $\mathcal{O}(TN)$ \\
Full Sequence & $(T, T, N)$ & $\mathcal{O}(T/N)$ & $\mathcal{O}(TN)$ \\
\midrule
\end{tabular}
\end{table*}

From Table~\ref{tab:asymp_error_rates}, we see that without buffering, the naive stochastic gradient has a $C_1 \cdot T/S$ term in the bias bound $\delta$.
The fully buffered subsequence and full sequence gradients remove the buffering error entirely, but require $\mathcal{O}(TN)$ computation.
Instead, our proposed buffered stochastic gradient controls the bias, with the geometrically decaying factor $L_\theta^B$, using only $\mathcal{O}(SN)$ computation.

\subsection{Buffering Error Bound for Nonlinear SSMs}
\label{subsub:buff_error_analysis}
To obtain a bound for the buffering error term (II),
we require the Lipschitz constant $L_\theta$ of smoothing kernels $\{\forwardkernel_t, \backwardkernel_t\}$ to be less than $1$.
Typically the smoothing kernels $\forwardkernel_t, \backwardkernel_t$ are not available in closed-form for nonlinear SSMs and therefore directly bounding the Lipschitz constant is difficult.
However, we now show that when the model's transition and emission densities are \emph{log-concave} in $x_t, x_{t-1}$, we can bound the Lipschitz constant of $\forwardkernel_t, \backwardkernel_t$ in terms of the Lipschitz constant of either the \emph{prior kernels} $\forwardkernel^{(0)}_t, \backwardkernel^{(0)}_t$, or the \emph{filtered kernels} $\forwardkernel^{(1)}_t, \backwardkernel^{(1)}_t$
\begin{align}
\nonumber
\forwardkernel_t^{(0)} &:=p(x_t \, | \, x_{t-1}, \theta), & \forwardkernel_t^{(1)} &:=p(x_t \, | \, x_{t-1}, y_t, \theta), & \\
\backwardkernel_t^{(0)} &:=p(x_t \, | \, x_{t+1}, \theta), & \backwardkernel_t^{(1)} &:=p(x_t \, | \, x_{t+1}, y_t, \theta), &
\end{align}
Unlike the smoothing kernels, the prior kernels are defined by the model and are therefore usually available.
If the filtered kernels are available, then they can be used to obtain even tighter bounds.

\begin{theorem}[Lipschitz Kernel Bound]
\label{thm:lipschitz_bound_logconcave}
Assume the prior for $x_{0}$ is log-concave in $x$.
If the transition density $p(x_t \, | \, x_{t-1}, \theta)$ is log-concave in $(x_t, x_{t-1})$ and the emission density $p(y_t \, | \, x_t)$ is log-concave in $x_t$,
then
\begin{align}
&\| \forwardkernel_t \|_{Lip} \, \leq \, \| \forwardkernel^{(1)}_t \|_{Lip} \, \leq \, \| \forwardkernel^{(0)}_t \|_{Lip} \\
&\| \backwardkernel_t \|_{Lip} \, \leq \, \| \backwardkernel^{(1)}_t \|_{Lip} \, \leq \, \| \backwardkernel^{(0)}_t \|_{Lip}.
\end{align}
Therefore
\begin{align}
\nonumber
L_\theta &= \max_t\{\|\forwardkernel_t\|_{Lip}, \|\backwardkernel_t\|_{Lip}\} \\
\nonumber
&\leq \max_t\{ \| \forwardkernel^{(1)}_t \|_{Lip}, \| \backwardkernel^{(1)}_t \|_{Lip}\} \\
&\leq \max_t\{ \| \forwardkernel^{(0)}_t \|_{Lip}, \| \backwardkernel^{(0)}_t \|_{Lip}\}
\end{align}
\end{theorem}
This theorem lets us bound $L_\theta$ with the Lipschitz constant of either the prior kernels or filtered kernels.
The proof of Theorem~\ref{thm:lipschitz_bound_logconcave} is provided in the Appendix and uses Caffarelli's log-concave perturbation theorem~\citep{villani2008optimal, colombo2017lipschitz}.
Examples of SSMs for which Theorem~\ref{thm:lipschitz_bound_logconcave} applies  include the linear Gaussian SSM, the stochastic volatility model, or any linear SSM with log-concave transition and emission distributions.

Theorem~\ref{thm:lipschitz_bound_logconcave} lets us calculate analytic bounds on $L_\theta$ for the buffering error of Theorem~\ref{thm:gradient_bias_mse}.
We provide explicit bounds for $L_\theta$ for the linear Gaussian SSM and stochastic volatility model in Section~\ref{sub:exp_models} with proofs in the Appendix.


\section{Experiments} \label{sec:exp}
We first empirically test the bias of our particle buffered gradient estimator $g^{\PF}_\theta$ on synthetic data for fixed $\theta$.
We then evaluate the performance of our proposed SGLD algorithm (Algorithm~\ref{alg:sgld}) on both real and synthetic data.

\subsection{Models}
\label{sub:exp_models}
For our experiments, we consider three models:
(i) the linear Gaussian SSM (LGSSM), a case where analytic buffering is possible, to assess the impact of the particle filter;
(ii) the stochastic volatility model (SVM) \citep{shephard2005}, where the emissions are non-Gaussian;
and (iii) the generalized autoregressive conditional heteroskedasticity (GARCH) model \citep{bollerslev1986}, where the latent transitions are nonlinear.
\subsubsection{Linear Gaussian SSM}
The \emph{linear Gaussian SSM} (LGSSM) is
\begin{align}
\nonumber
X_t \, | \, (X_{t-1} = x_{t-1}, \theta) &\sim \mathcal{N}(x_t \, |\, \phi x_{t-1} \, ,\, \sigma^2), \\
Y_t \, | \, (X_t = x_t, \theta) &\sim \mathcal{N}(y_t \, |\, x_t \, ,\, \tau^2), 
\end{align}
with $\nu_0(x_0) = \mathcal{N}(x_0 \, | \, 0, \frac{\phi^2}{1-\sigma^2})$ and parameters $\theta = (\phi, \sigma, \tau)$.

The transition and emission distributions are both Gaussian and log-concave in $x$, so Theorem~\ref{thm:lipschitz_bound_logconcave} applies.
In the Appendix, we show that the filtered kernels 
of the LGSSM are bounded with the Lipschitz constant $L_\theta = |\phi| \cdot \sigma^2/(\sigma^2+\tau^2)$.
Thus, the buffering error decays geometrically with increasing buffer size $B$ when $|\phi| < (1 + \frac{\tau^2}{\sigma^2})$.
This linear model serves as a useful baseline since the various terms in \eqref{eq:error_decompose} can be calculated analytically.

\subsubsection{Stochastic Volatility Model}
The \textit{stochastic volatility model} (SVM) is
\begin{align}
\nonumber
X_t \, | \, (X_{t-1}=x_{t-1}, \theta) &\sim \mathcal{N}(x_t \,|\, \phi x_{t-1}\,,\, \sigma^2), \\
Y_t \, | \, (X_t = x_t, \theta) &\sim \mathcal{N}(y_t \,|\, 0\,,\, \exp(x_t)\tau^2),
\label{eq:svmend}
\end{align}
with $\nu_0(x_0) = \mathcal{N}(x_0 \, | \, 0, \frac{\phi^2}{1-\sigma^2})$ and parameters $\theta = (\phi, \sigma, \tau)$. \par
For the SVM, the transition and emission distributions are log-concave in $x$, allowing  Theorem~\ref{thm:lipschitz_bound_logconcave} to apply.
In the Appendix, we show that the prior kernels $\{\forwardkernel_t^{(0)}, \backwardkernel_t^{(0)}\}$ of the SVM are bounded with the Lipschitz constant $L_\theta = |\phi|$.
Thus, the buffering error decays geometrically with increasing buffer size $B$ when $|\phi| < 1$.

\subsubsection{GARCH Model}
We finally consider a GARCH(1,1) model (with noise)
\begin{align}
\nonumber
&X_t \, | \, (X_{t-1}=x_{t-1}, \sigma_t^2, \theta) \sim \mathcal{N}(x_t \,|\, 0,\, \sigma_t^2), \\
\nonumber
&\sigma_t^2(x_{t-1}, \sigma_{t-1}^2, \theta) = \alpha + \beta x_{t-1}^2 + \gamma \sigma_{t-1}^2, \\
&Y_t \, | \, (X_t = x_t, \theta) \sim \mathcal{N}(y_t \,|\, x_t\,,\,\tau^2),
\end{align}
with $\nu_0(x_0) = \mathcal{N}(0, \frac{\alpha}{1-\beta-\gamma})$ and parameters $\theta = (\alpha, \beta, \gamma, \tau)$.
Unlike the LGSSM and SVM, the noise between $X_t$ and $X_{t-1}$ is multiplicative in $X_{t-1}$ rather than additive.
This model's transition distribution is \emph{not} log-concave in ($x_t, x_{t-1}$) and therefore our theory (Theorem~\ref{thm:lipschitz_bound_logconcave}) does not hold.
However, we see empirically that buffering can help reduce the gradient error for the GARCH in the experiments below and in the Appendix.

\subsection{Stochastic Gradient Bias}
\label{sub:exp_grad}
\begin{figure*}[ht]
\begin{center}
\begin{minipage}[c]{.365\textwidth}
    \centering
        \includegraphics[width=\textwidth, trim=0.09in 0 0 0, clip]{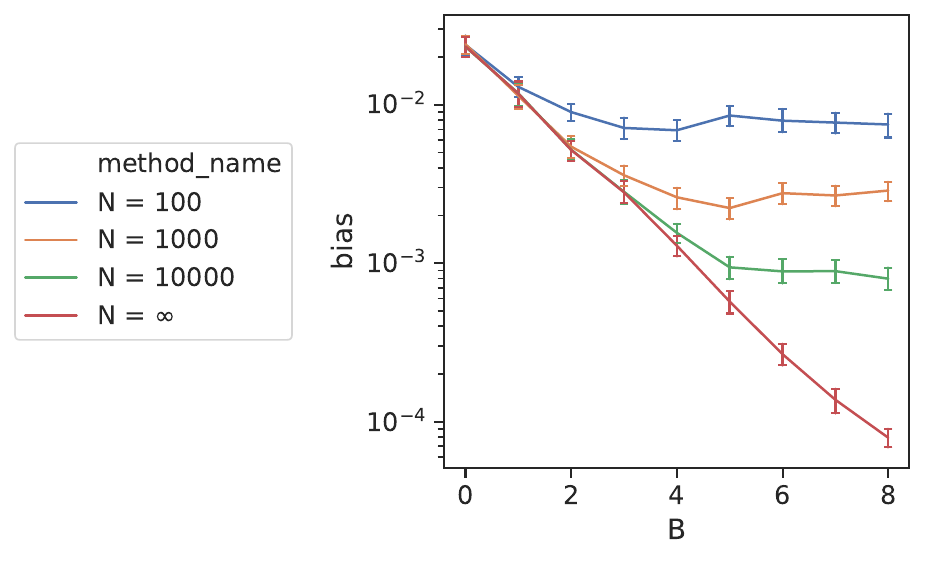}
\end{minipage}
\begin{minipage}[c]{.25\textwidth}
    \centering
        \includegraphics[width=\textwidth, trim=2in 0 0 0, clip]{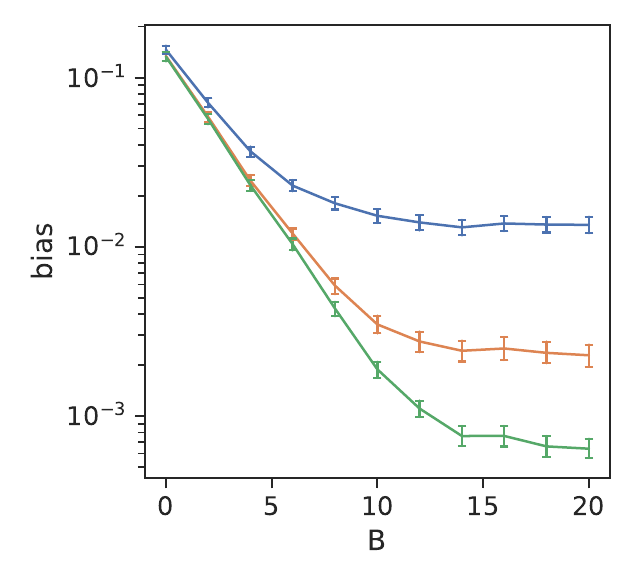}
\end{minipage}
\begin{minipage}[c]{.25\textwidth}
    \centering
        \includegraphics[width=\textwidth, trim=2in 0 0 0, clip]{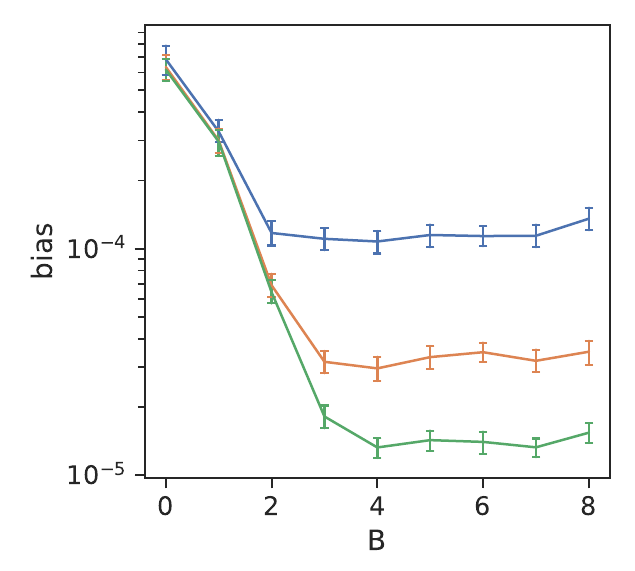}
\end{minipage}
\vspace{-0.5em}
\caption{Stochastic gradient bias varying buffer size $B$ for $S = 16$ for different values of $N$. (left) LGSSM $\phi$, (middle) SVM $\phi$, (right) GARCH $\beta$. Error bars are 95\% confidence interval over 1000 replications.}
\label{fig:varyB}
\end{center}
\vskip -0.2in
\end{figure*}

\begin{figure*}[ht]
\begin{center}
\begin{minipage}[c]{.365\textwidth}
    \centering
        \includegraphics[width=\textwidth, trim=0.09in 0 0 0, clip]{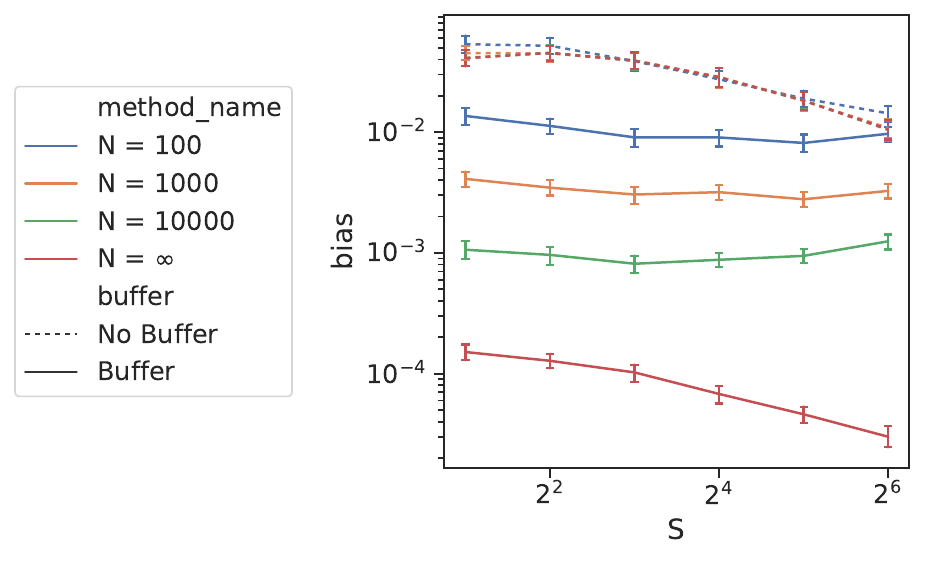}
\end{minipage}
\begin{minipage}[c]{.25\textwidth}
    \centering
        \includegraphics[width=\textwidth, trim=2in 0 0 0, clip]{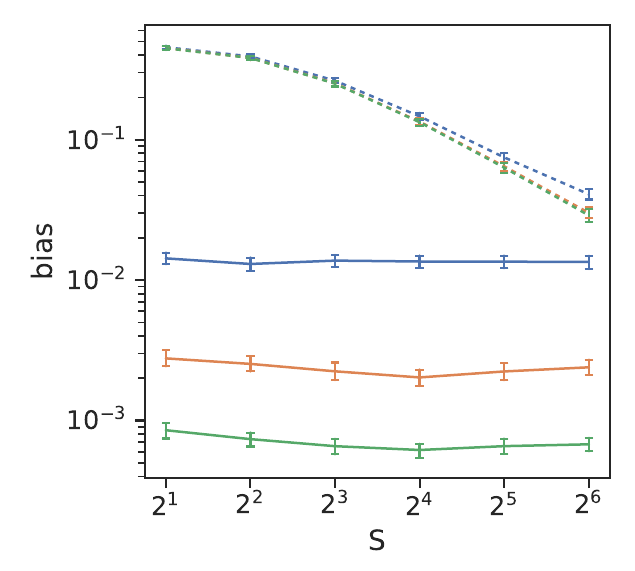}
\end{minipage}
\begin{minipage}[c]{.25\textwidth}
    \centering
        \includegraphics[width=\textwidth, trim=2in 0 0 0, clip]{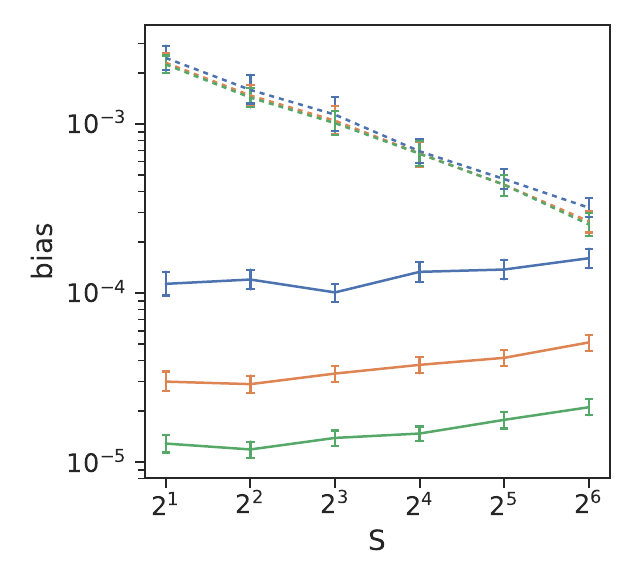}
\end{minipage}
\vspace{-0.5em}
\caption{Stochastic gradient bias varying subsequence size $S$ for No Buffer ($B = 0$) and Buffer ($B > 0$) for different values of $N$. (left) LGSSM $\phi$, (middle) SVM $\phi$, (right) GARCH $\beta$. The buffer size $B = 8$ for LGSSM and GARCH and $B = 16$ for the SVM. Error bars are 95\% confidence interval over 1000 replications.}
\label{fig:varyS}
\end{center}
\vskip -0.2in
\end{figure*}
\begin{figure*}[ht]
\begin{center}
\begin{minipage}[c]{.365\textwidth}
    \centering
        \includegraphics[width=\textwidth, trim=0.09in 0 0 0, clip]{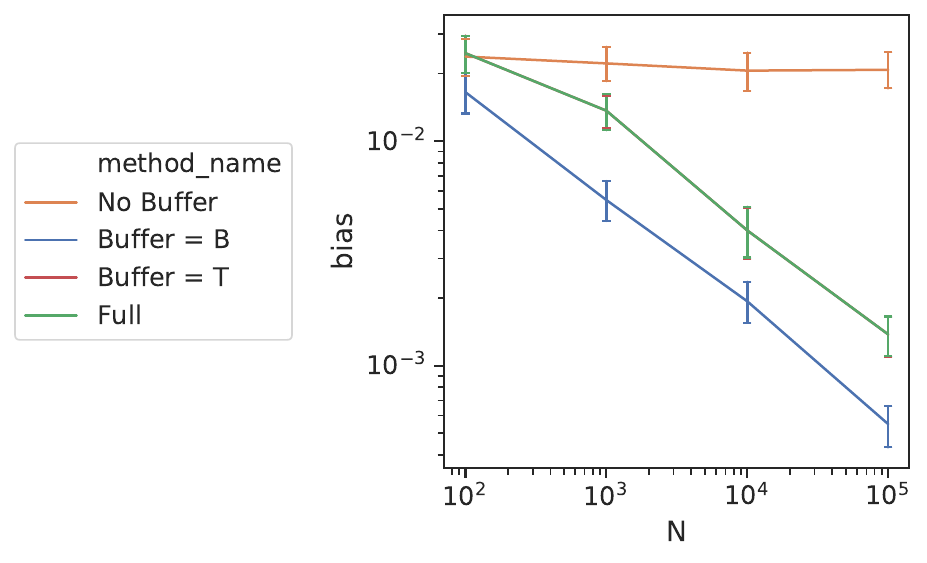}
\end{minipage}
\begin{minipage}[c]{.25\textwidth}
    \centering
        \includegraphics[width=\textwidth, trim=2in 0 0 0, clip]{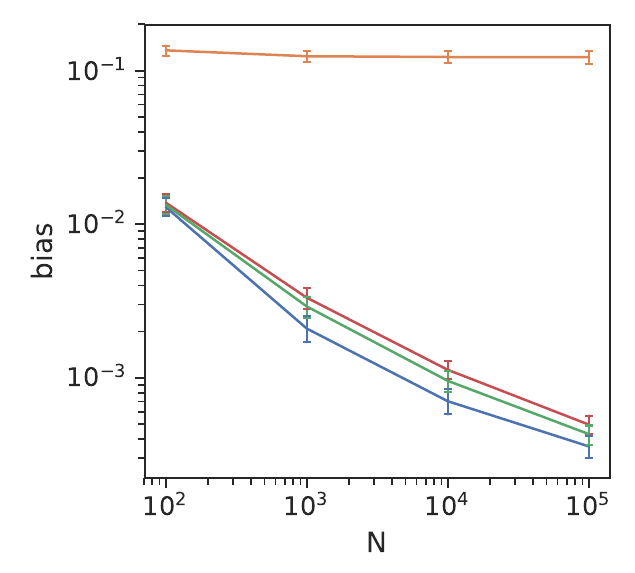}
\end{minipage}
\begin{minipage}[c]{.25\textwidth}
    \centering
        \includegraphics[width=\textwidth, trim=2in 0 0 0, clip]{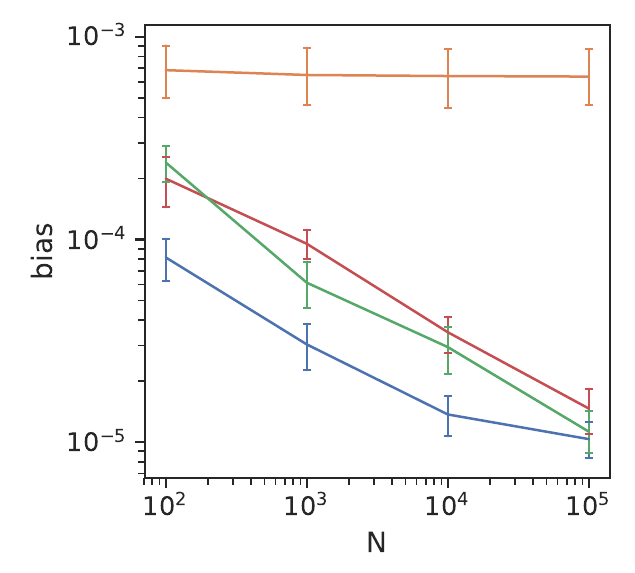}
\end{minipage}
\vspace{0.5em}

\begin{minipage}[c]{.365\textwidth}
    \centering
        \includegraphics[width=\textwidth, trim=0.09in 0 0 0, clip]{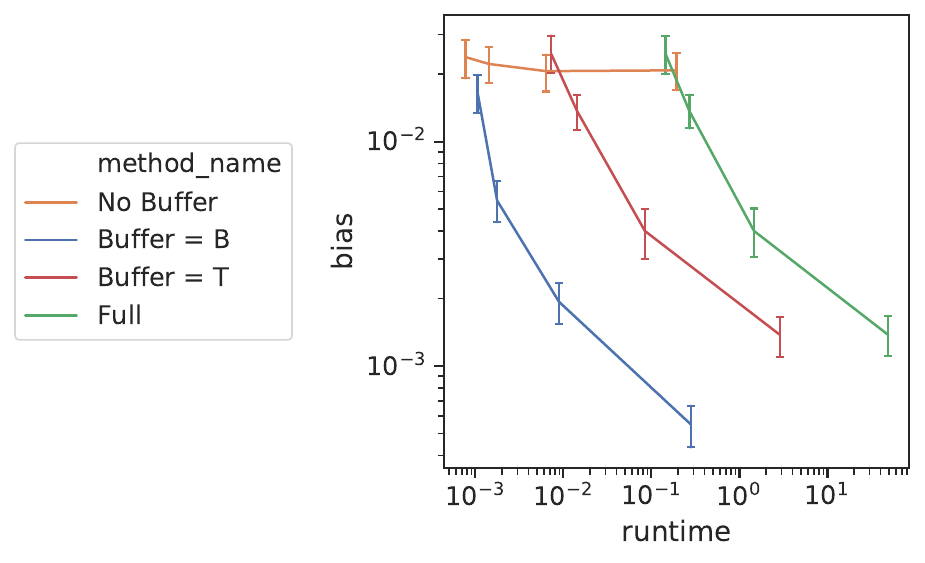}
\end{minipage}
\begin{minipage}[c]{.25\textwidth}
    \centering
        \includegraphics[width=\textwidth, trim=2in 0 0 0, clip]{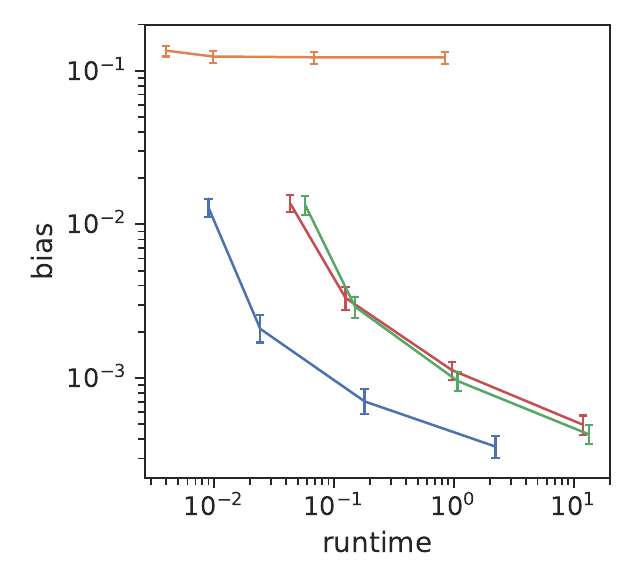}
\end{minipage}
\begin{minipage}[c]{.25\textwidth}
    \centering
        \includegraphics[width=\textwidth, trim=2in 0 0 0, clip]{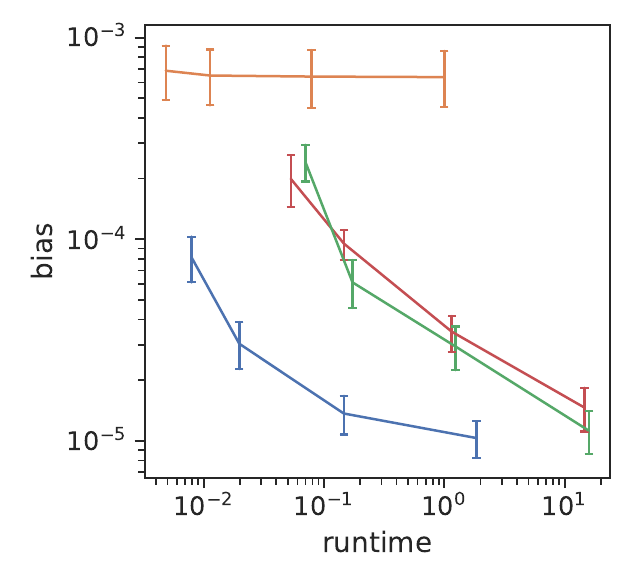}
\end{minipage}
\vspace{-0.5em}
\caption{Stochastic gradient bias varying $N$ for different $S,B$. (left) LGSSM $\phi$, (middle) SVM $\phi$, (right) GARCH $\beta$. (top) $x$-axis is $N$, (bottom) $x$-axis is runtime in seconds. \textcolor{orange}{\textbf{No Buffer}} is $g^{\PF}(16,0,N)$, \textcolor{blue}{\textbf{Buffer} $B = B$} is $g^{\PF}(16,B,N)$, \textcolor{BrickRed}{\textbf{Buffer} $B = T$} is $g^{\PF}(16,T,N)$, and \textcolor{ForestGreen}{\textbf{Full}} is $g^{\PF}(T,T,N)$. The moderate buffer size $B = 8$ for LGSSM and GARCH and $B = 16$ for the SVM. 
Error bars are 95\% confidence interval over 1000 replications.
}
\label{fig:varyTime}
\end{center}
\vskip -0.2in
\end{figure*}

We compare the error of stochastic gradient estimates using a buffered subsequence with $S = 16$, while varying $B$ and $N$ on synthetic data from each model.
We generated synthetic data of length $T=256$ using $(\phi = 0.9, \sigma = 0.7, \tau = 1.0)$ for the LGSSM,
$(\phi = 0.9, \sigma = 0.5, \tau=0.5)$ for the SVM,
and $(\alpha = 0.1, \beta = 0.8, \gamma = 0.05, \tau = 0.3)$ for the GARCH model.

Figures~\ref{fig:varyB}-\ref{fig:varyTime} display the bias of our particle buffered stochastic gradient $g^{\PF}_\theta(S, B, N)$ and $g_\theta$ averaged over 1000 replications.
We evaluate the gradients at $\theta$ equal to the data generating parameters.
We vary the buffer size $B \in [0, 16]$, the subsequence size $S \in [1,T]$ and the number of samples $N \in \{100, 1000, 10000\}$.
For the LGSSM, we also consider $N = \infty$, by calculating $g^{\PF}_\theta(S, B, \infty)$ using the Kalman filter \citep{kalman1960}, which is tractable in the linear setting.
We calculate $g_\theta$ using the Kalman filter for the LGSSM, and use $g_\theta \approx g^{\PF}_\theta(T,0,10^7)$ for the SVM and the GARCH model, assuming that $N = 10^7$ particles is sufficient for an accurate approximation in these 1-dimensional settings.

Figure~\ref{fig:varyB} shows the bias as we vary the buffer size $B$ for different $N$ and $S = 16$.
From Figure~\ref{fig:varyB}, we see the trade-off between the buffering error (II) and the particle error (III) in the bias bound, \eqref{eq:grad_bias_bound} of Theorem~\ref{thm:gradient_bias_mse}.
For all $N$, when $B$ is small, the buffering error (II) dominates, and therefore the MSE decays exponentially as $B$ increases.
However for $N < \infty$, the particle error (III) dominates for larger values of $B$.
In fact, the bias slightly increases due to particle degeneracy, as $|\BUFFSUBSEQ| = S+2B$ increases with $B$.
For $N = \infty$ in the LGSSM case, we see that the bias continues to decreases exponentially with large $B$ as there is no particle filter error when using the Kalman filter.

Figure~\ref{fig:varyS} shows the bias as we vary the subsequence size $S$ for different $N$ and with and without buffering.
We see that buffering helps regardless of subsequence size (as the bias for all buffered methods are lower than the no buffer methods for all $S \in [2, 64]$).
We also see that increasing $S$ can increase the bias for fixed $N$ (when buffering) as the particle error (III) dominates.

Figure~\ref{fig:varyTime} shows the bias as we vary the number of particles $N$ for the four different methods correspond to Table~\ref{tab:asymp_error_rates}.
In the top row, we compare the bias against $N$ and in the bottom row, we compare the bias against the runtime required to calculate $g^{\PF}_\theta$.
We see that the method without buffering (\textcolor{orange}{orange}) is significantly biased regardless of $N$, where as buffering with moderate $B$ (\textcolor{blue}{blue}), buffering with large $B = T$ (\textcolor{BrickRed}{red}), and using the full sequence (\textcolor{ForestGreen}{green}) have similar (lower) bias as we increase $N$.
However the runtime plots show that buffering with moderate $B$ takes significantly less time.

In summary, Figures~\ref{fig:varyB}-\ref{fig:varyTime} show that buffering cannot be ignored in these three example models: there is high bias for $B = 0$.
In general, buffering has diminishing returns when $B$ is excessively large relative to $N$.

In the Appendix, we present plots of the bias varying $B,S,N$ using different particle filters (PaRIS and Poyiadjis $N^2$) instead of the naive PF.
We find that they perform similarly to the naive PF for the small subsequence lengths $|\BUFFSUBSEQ|$ considered, while taking $\approx 10$ times longer to run.
We also present plots of the bias as we vary the parameters of the data generating model. We find that as the parameters become more challenging (e.g. $L_\theta \rightarrow 1$), we need to increase both $B$ and $N$ to control bias; otherwise, the buffer stochastic subsequence methods are more biased than using full sequence gradient.  

\subsection{SGLD Experiments}
\label{sub:exp_sgld}
Having examined the stochastic gradient bias, we now examine using our buffered stochastic gradient estimators in SGLD (Algorithm~\ref{alg:sgld}).

\subsubsection{SGLD Evaluation Method}

We measure the sample quality of our MCMC chains $\{\theta^{(k)}\}_{k=1}^K$ using the \emph{kernel Stein discrepancy} (KSD) for equal compute time~\citep{gorham2017measuring, liu2016kernelized}.
We choose to use KSD rather than classic MCMC diagnostics such as effective sample size (ESS)~\citep{gelman2013bayesian}, because KSD penalizes the bias present in our MCMC chains. Whilst it can be hard to interpret the absolute value of KSD for any problem, it is informative for comparing between different algorithms.
Given a sample chain (after burnin and thinning) $\{\theta^{(k)}\}_{k=1}^{\tilde{K}}$,
let $\hat{p}(\theta | y)$ be the empirical distribution of the samples.
Then the KSD between $\hat{p}(\theta | y)$ and the posterior distribution $p(\theta | y)$ is
\begin{equation}
\KSD(\hat{p},p) = \sum_{d = 1}^{\text{dim}(\theta)} \sqrt{\sum_{k,k' = 1}^{\tilde{K}} \frac{\mathcal{K}_0^d(\theta^{(k)}, \theta^{(k')})}{\tilde{K}^2}},
\end{equation}
where
\begin{equation}
\label{eq:ksd_k0}
\mathcal{K}_0^d(\theta, \theta') = \tfrac{1}{p(\theta|y) p(\theta'|y)} \grad_{\theta_d} \grad_{\theta_d'} (p(\theta|y) \mathcal{K}(\theta, \theta') p(\theta' | y))
\end{equation}
and $\mathcal{K}(\cdot, \cdot)$ is a valid kernel function.
Following~\citet{gorham2017measuring}, we use the inverse multiquadratic kernel $\mathcal{K}(\theta,\theta') = (1+\|\theta-\theta'\|_2^2)^{-0.5}$ in our experiments.
Since \eqref{eq:ksd_k0} requires full gradient evaluations of $\log p(\theta|y)$ that are computationally intractable, we replace these terms with corresponding stochastic estimates using the full particle filter estimate, $g^{\PF}_\theta$~\cite{gorham2020stochastic}. 


\subsubsection{SGLD on Synthetic LGSSM Data}
\label{subsub:LGSSM_synth_sgld}
To assess the effect of using particle filters with buffered stochastic gradients,
we first focus on SGLD on synthetic LGSSM data, where calculating $\widehat{g}_\theta(S,B)$ is possible.
We generate training sequences of length $T = 10^3$ or $10^6$ using the same parametrization as Section~\ref{sub:exp_grad}.

We consider three pairs of different gradient estimators: \textcolor{ForestGreen}{\textbf{Full}} $(S = T)$, \textcolor{Blue}{\textbf{Buffered}} $(S = 40, B = 10)$ and \textcolor{orange}{\textbf{No Buffer}} $(S = 40, B = 0)$ each with $N = 1000$ particles using the particle filter and with $N = \infty$ using the Kalman filter.
To select the stepsize, we performed a grid search over $\epsilon \in \{1, 0.1, 0.01, 0.001\}$ and selected the method with smallest KSD to the posterior on the training set.
We present the KSD results (for the best $\epsilon$) in Table~\ref{tab:lgssm_ksd} and trace plots of the metrics in Figure~\ref{fig:lgssm-sgld}.

\begin{figure*}[ht]
\vskip 0.2in
\begin{center}
\begin{minipage}[c]{.49\textwidth}
\centering
(a) $T = 10^3$ \\
\begin{minipage}[c]{\textwidth}
    \centering
        \includegraphics[width=\textwidth]{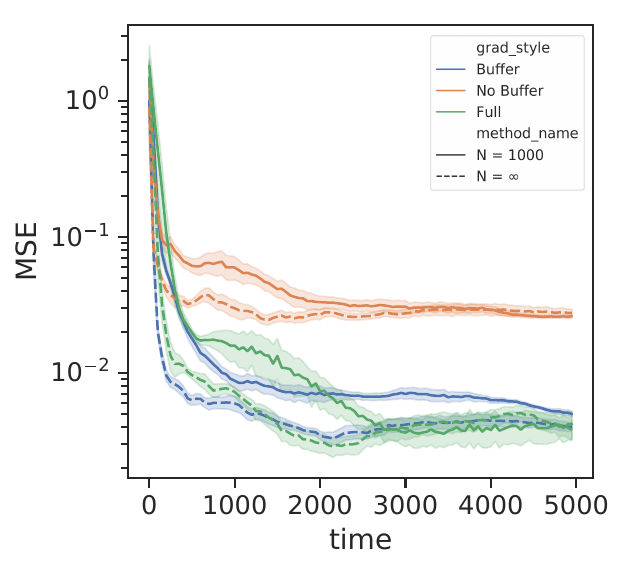}
\end{minipage}
\end{minipage}
\begin{minipage}[c]{.49\textwidth}
\centering
(b) $T = 10^6$ \\
\begin{minipage}[c]{\textwidth}
    \centering
        \includegraphics[width=\textwidth]{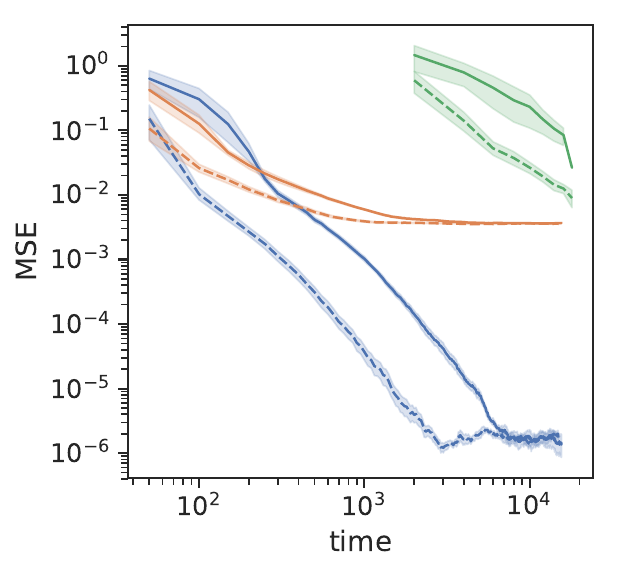}
\end{minipage}
\end{minipage}
\caption{Comparison of SGLD with different gradient estimates on synthetic LGSSM data: $T = 10^3$ (left), $T = 10^6$ (right). 
MSE of estimated posterior mean to true $\phi=0.9$.}
\label{fig:lgssm-sgld}
\end{center}
\vskip -0.2in
\end{figure*}

From Figure~\ref{fig:lgssm-sgld}, we see that the methods without buffering ($B = 0$) 
have higher MSE as they are biased. 
We also see that the full sequence methods ($S = T$) perform poorly for large $T = 10^6$.

The KSD results further support this story.
Table~\ref{tab:lgssm_ksd} presents the mean and standard deviation on our estimated $\log_{10}$ KSD for $\theta$. 
Tables of the marginal KSD for individual components of $\theta$ can be found in the Appendix.
The methods without buffering have larger KSD, as the inherent bias of $\widehat{g}_\theta(S, B=0)$ led to an incorrect stationary distribution.
The full sequence methods perform poorly for $T=10^6$ because of a lack of samples that can be computed in a fixed runtime.

\begin{table}[ht]
\caption{KSD for Synthetic LGSSM. Mean and SD. Results are shown after running each method for a fixed computational time.}
\label{tab:lgssm_ksd}
\vskip 0.15in
\begin{center}
\begin{small}
\begin{sc}
\begin{tabular}{lll|l|l}
\toprule
    &     &        & \multicolumn{2}{c}{$\log_{10}$KSD} \\
$S$ & $B$ & $N$ & $T = 10^3$ & $T = 10^6$ \\
\midrule
$T$ & --  & $1000$    &  0.85 (0.08)           &  4.92 (0.40) \\
    &     & $\infty$  &  \textbf{0.64 (0.17) } &  4.85 (0.36) \\
 \midrule
 40 &  0  & $1000$    &  1.58 (0.03)           &  4.68 (0.10) \\
    &     & $\infty$  &  1.55 (0.03)           &  4.68 (0.11) \\
 \midrule
 40 &  10 & $1000$    &  \textbf{0.68 (0.25)}  &  \textbf{3.43 (0.19)} \\
    &     & $\infty$  &  \textbf{0.61 (0.21)}  &  \textbf{3.25 (0.29)} \\
\bottomrule
\end{tabular}\end{sc}
\end{small}
\end{center}
\vskip -0.1in
\end{table}

In the Appendix, we present similar results on synthetic SVM and GARCH data.
Also in the Appendix, we present results on LGSSM in higher dimensions.
As is typical in the particle filtering literature,
the performance degrades with increasing dimensions for $N$ fixed.

\subsubsection{SGLD on Exchange Rate Log-Returns}
\label{sub:exchange}
We now consider fitting the SVM and the GARCH model to EUR-USD exchange rate data at the minute resolution from November 2017 to October 2018.
The data consists of 350,000 observations of demeaned log-returns.
As the market is closed during non-business hours, we further break the data into 53 weekly segments of roughly 7,000 observations each.
In our model, we assume independence between weekly segments and divide the data into a training set of the first 45 weeks and a test set of the last 8 weeks. 
Full processing details and example plots are in the Appendix.
Our method (Algorithm~\ref{alg:sgld}) easily scales to the unsegmented series; however the abrupt changes between starts of weeks are not adequately modeled by \eqref{eq:svmend}

We fit both the SVM and the GARCH model using SGLD with four different gradient methods:
(i) \textcolor{ForestGreen}{\textbf{Full}}, the full gradient over all segments in the training set;
(ii) \textcolor{BrickRed}{\textbf{Weekly}}, a stochastic gradient over a randomly selected segment in the training set;
(iii) \textcolor{orange}{\textbf{No Buffer}}, a stochastic gradient over a randomly selected \emph{subsequence} of length $S = 40$;
and (iv) \textcolor{Blue}{\textbf{Buffer}}, our buffered stochastic gradient for a subsequence of length $S = 40$ with buffer length $B = 10$.
To estimate the stochastic gradients, we use Algorithm~\ref{alg:pf} with $N = 1000$.
To select the stepsize parameter, we performed a grid search over $\epsilon \in \{1, 0.1, 0.01, 0.001\}$ and selected the method with smallest KSD.
We present the KSD results in Table~\ref{tab:exchange_ksd}.



\begin{table}[ht]
\caption{KSD for SGLD on exchange rate data. Mean and SD over 5 chains each. Results are shown after running each method for a fixed computational time.} 
\label{tab:exchange_ksd}
\vskip 0.15in
\begin{center}
\begin{small}
\begin{sc}
\begin{tabular}{l|l|l}
\toprule
        & \multicolumn{2}{c}{$\log_{10}$KSD} \\
 Method & SVM & GARCH \\
 \midrule
\textcolor{ForestGreen}{\textbf{Full}}     & 4.03 (0.14)           &   2.84 (0.30)\\
\textcolor{BrickRed}{\textbf{Weekly}}    & 3.87 (0.08)           &   2.81 (0.21) \\
\textcolor{orange}{\textbf{No Buffer}} & 4.48 (0.01)           &   \textbf{2.09 (0.09)}\\
\textcolor{Blue}{\textbf{Buffer}}    &  \textbf{3.56 (0.08)} &   \textbf{2.19 (0.05)}\\
\bottomrule
\end{tabular}\end{sc}
\end{small}
\end{center}
\vskip -0.1in
\end{table}

For the SVM,
we see that buffering leads to more accurate MCMC samples, Table~\ref{tab:exchange_ksd}~(left).
In particular, the samples from SGLD without buffering have smaller $\phi, \tau^2$ and a larger $\sigma^2$, indicating that its posterior is (inaccurately) centered around a SVM with larger latent state noise.
We also again see that the full sequence and weekly segment methods perform poorly due to the limited number of samples that can be computed in a fixed runtime.

For the GARCH model, Table~\ref{tab:exchange_ksd}~(right), we see that the subsequence methods out perform the full sequence methods, but unlike in the SVM, buffering does not help with inference on the GARCH data.
This is because the GARCH model that we recover on the exchange rate data (for all gradient methods) is close to white noise $\beta \approx 0$.
Therefore the model believes the observations are close to independent, hence no buffer is necessary.

\section{Discussion}
In this work, we developed a particle buffered stochastic gradient estimators for nonlinear SSMs. Our key contributions are (i) extending buffered stochastic gradient MCMC with particle filtering for nonlinear SSMs, (ii) analyzing the error of our proposed particle buffered stochastic gradient $g^{\PF}_\theta$ (Theorem~\ref{thm:gradient_bias_mse}) and its affect on our SGLD Algorithm~\ref{alg:sgld} (Theorem~\ref{thm:finite_sample_error}), and (iii) generalizing the geometric decay bound for buffering to nonlinear SSMs with log-concave likelihoods (Theorem~\ref{thm:lipschitz_bound_logconcave}). We evaluated our proposed gradient estimator with SGLD on both synthetic data and EUR-USD exchange rate data. We find that buffering is necessary to control bias and that our stochastic gradient methods (Algorithm~\ref{alg:sgld}) are able to out perform batch methods on long sequences.

Possible future extensions of this work include relaxing the log-concave restriction of Theorem~\ref{thm:lipschitz_bound_logconcave}, extensions to Algorithm~\ref{alg:sgld} as discussed at the end of Section~\ref{sub:sgld_alg}, and applying our particle buffered stochastic gradient estimates to other applications than SGMCMC, such as maximising loglikelihoods or optimization in variational autoencoders for sequential data~\citep{maddison2017filtering, naesseth2018variational}.

\subsection*{Acknowledgements}
We would like to thank Nicholas Foti for helpful discussions.
This work was supported in part by:
ONR Grants N00014-15-1-2380, N00014-18-1-2862, and N00014-22-1-2110; 
NSF CAREER Award IIS-1350133; AFOSR Grant FA9550-21-1-0397; and, 
EPSRC Grants EP/L015692/1, EP/S00159X/1, EP/V022636/1, EP/R01860X/1, EP/R018561/1 and EP/R034710/1.



\bibliographystyle{plainnat}
\bibliography{references}

\clearpage
\setcounter{section}{0}
\renewcommand{\thesection}{\Alph{section}}
\renewcommand{\theequation}{\Alph{section}.\arabic{equation}}
\numberwithin{equation}{section}
\numberwithin{theorem}{section}
\numberwithin{lemma}{section}
\numberwithin{figure}{section}
\numberwithin{table}{section}

\section*{Appendix}
\noindent This Appendix is organized as follows.
In Section~\ref{supp-sec:error}, we provide additional details and proofs for the error analysis of Section~\ref{sub:error}.
In particular, we provide the proof of Theorem~\ref{thm:finite_sample_error} in Section~\ref{supp-subsub:finite_sample_error},
the proof of Theorem~\ref{thm:gradient_bias_mse} in Section~\ref{supp-subsub:graderror},
the proof of Theorem~\ref{thm:lipschitz_bound_logconcave} in Section~\ref{supp-subsub:lipschitzproof} and applications of Theorem~\ref{thm:lipschitz_bound_logconcave} for LGSSM and SVM in Section~\ref{supp-sub:model_bounds}.
In Section~\ref{supp-sec:models}, we provide additional particle filter and gradient details for the models in Section~\ref{sub:exp_models}.
In Section~\ref{supp-sec:experiments}, we provide additional details and figures of experiments.

\section{Error Analysis Proofs}
\label{supp-sec:error}

In this section, we provide additional details and proofs for the error analysis of Section~\ref{sub:error}.
In particular, we provide the proof of Theorem~\ref{thm:finite_sample_error} in Section~\ref{supp-subsub:finite_sample_error},
the proof of Theorem~\ref{thm:gradient_bias_mse} in Section~\ref{supp-subsub:graderror},
the proof of Theorem~\ref{thm:lipschitz_bound_logconcave} in Section~\ref{supp-subsub:lipschitzproof} and applications of Theorem~\ref{thm:lipschitz_bound_logconcave} for LGSSM and SVM in Section~\ref{supp-sub:model_bounds}.

\subsection{Proof of Theorem~\ref{thm:finite_sample_error}}
\label{supp-subsub:finite_sample_error}
We now prove the error bounds for biased SGLD's finite sample average found in Section~\ref{subsub:finite_sample_error}.
The proof is a modification of the proof of Theorem 3 found in Supplement~E of \citet{chen2015convergence}.

Recall our assumption on $\phi$ is that $\psi(\theta)$ and its derivatives are bounded by some finite constant $M$ (Assumption~1).
This is the implicit moment condition for $\phi$, which is also assumed by~\citet{vollmer2016non} and~\citet{chen2015convergence}.

The proof of Theorem~\ref{thm:finite_sample_error} then proceeds as in Theorem 3 of~\citet{chen2015convergence}, except that we allow for a $\delta > 0$ such that $\E \| \hat{g}(\theta) - g(\theta) \| \leq \delta$ for all $\theta$ rather than restrict $\delta = 0$.

For compactness of notation, we will use $g_k$ to denote $g(\theta^{(k)})$, $\hat{g}_k$ to denote $\hat{g}(\theta^{(k)})$, and $\psi_k$ to denote $\psi(\theta^{(k)}).$

\begin{proof}[Proof of Theorem~\ref{thm:finite_sample_error}]
Following~\citet{chen2015convergence}, from the definition of the functional $\psi$ and generator $\mathcal{L}$, we have
\begin{align}
\hat{\phi}_{K,\epsilon} - \bar{\phi} =& \frac{\E\psi_K - \psi_1}{K \epsilon} - \frac{\sum_{k = 1}^K \left(\E\psi_k - \psi_k\right)}{K \epsilon} + \frac{\sum_{k=1}^K (\hat{g}_k - g_k) \cdot \grad \psi_k}{K}+ \mathcal{O}(\epsilon) \enspace,
\label{eq:sgld_mse_proof_diff}
\end{align}
and $\E(\E\psi_k - \psi_k)^2$ is $\mathcal{O}(\epsilon)$.
Because $\psi$ is bounded by $M$, we also have $\E\psi_K - \E\psi_1 < 2M$ and $\E (\E\psi_K - \psi_1)^2 < 4M^2$.

Let $\xi_k = (\hat{g}_k - g_k) \cdot \grad \psi_k$.
From our assumptions on the bias and MSE of $\hat{g}$ and as $\grad\psi$ is bounded, we have $|\E \xi_k | \leq M\delta$ and $\E[\xi_k^2] \leq M^2\sigma^2$ for all $k$.
In addition, we have for all $k \neq k'$
\begin{equation}
|\E[\xi_k \xi_{k'}] | \leq M^2 \cdot \|\E[\hat{g}_k - g_k] \| \cdot  \|\E[\hat{g}_{k'} - g_{k'}] \| \leq M^2 \delta^2 \enspace,
\end{equation}
where the expectations are over \emph{independent} stochastic subsequences $\SUBSEQ$ chosen at steps $k$ and $k'$.

To prove the bias bound, we take the expectation of \eqref{eq:sgld_mse_proof_diff}, let $C = 2M$ and bound each term
\begin{equation}
| \E\hat\phi_{K,\epsilon} - \bar\phi | \leq \frac{2M}{K \epsilon} + M\delta + \mathcal{O}\left( \epsilon\right) \leq C \cdot\left( \frac{1}{K\epsilon} + \delta\right) + \mathcal{O}(\epsilon) \enspace.
\end{equation}

To prove the MSE bound, we take the square and expectation of both sides of \eqref{eq:sgld_mse_proof_diff},
\begin{align}
\nonumber
\E(\hat{\phi}_{K,\epsilon} - \bar{\phi}\,)^2 \leq & \\
\nonumber
&\mkern-36mu
\E\Bigg[\frac{(\E\psi_K - \psi_0))^2}{K^2 \epsilon^2} + \frac{\sum_{k = 1}^K \left(\E\psi_k - \psi_k\right)^2}{K^2 \epsilon^2} +\frac{\sum_{k=1}^K \xi_{k}^2 + \sum_{k \neq k' = 1}^K \xi_{k} \xi_{k'}}{K^2}\\
&\mkern-18mu
+\frac{\sum_{k=1}^K \xi_{k}}{K} \cdot \Bigg(\frac{\E\psi_K - \psi_1}{K\epsilon} - \frac{\sum_{k = 1}^K \left(\E\psi_k - \psi_k\right)}{K \epsilon} + \mathcal{O}\left(\epsilon\right) 
\Bigg) 
+ \mathcal{O}(\epsilon^2)
\Bigg] .
\label{eq:sgld_mse_proof_mse}
\end{align}
The first two lines are the squared terms and the last two lines are the cross terms that do not go to zero.
In particular, we do not assume $\hat{g}(\theta)$ is unbiased for $g(\theta)$, therefore we keep the cross-terms involving $\xi_k$.
Bounding each term of \eqref{eq:sgld_mse_proof_mse} gives the MSE bound
\begin{align}
\nonumber
\E(\hat{\phi}_{K,\epsilon} - \bar{\phi}\,)^2 &\leq \frac{4M^2}{K^2 \epsilon^2} + \frac{K \cdot \mathcal{O}(\epsilon)}{K^2 \epsilon^2} + \frac{K M^2\sigma^2 + K^2 M^2 \delta^2}{K^2} \\
\nonumber
&\quad\quad
+ M\delta \cdot \left(\frac{2M}{K\epsilon} + \frac{K \cdot 2M}{K \epsilon} + \mathcal{O}\left(\epsilon\right) \right) + \mathcal{O}(\epsilon^2) \\
&\leq
C \cdot \left(\frac{1}{K^2 \epsilon^2} + \frac{\sigma^2}{K} + \delta^2 + \frac{\delta}{\epsilon}\right) + \mathcal{O}\left(\frac{1}{K \epsilon} + \delta \epsilon + \epsilon^2\right) \enspace.
\label{eq:sgld_mse_proof_bound}
\end{align}
\end{proof}

\subsection{Proof of Theorem~\ref{thm:gradient_bias_mse}}
\label{supp-subsub:graderror}
We now prove Theorem~\ref{thm:gradient_bias_mse}, which bounds the bias and MSE of our buffered stochastic gradient $g^{\PF}_\theta(S,B,N)$.
In our proof, we use Lemma~\ref{lemma:suberror} to bound the subsequence error which is proved in Section~\ref{supp-subsub:suberror}.

\begin{proof}[Proof of Theorem~\ref{thm:gradient_bias_mse}]
For the bias bound, \eqref{eq:grad_bias_bound}, wwe apply the triangle inequality to decompose the error into three terms 
\begin{align}
\label{eq:bias_decompose}
\| \E g^{\PF}_\theta(S, &B, N) - g_\theta \| \leq & \\
\nonumber
&\mkern-36mu \underbrace{\| \E(g^{\PF}_\theta(S,B,N) - \hat{g}_\theta(S,B)) \|}_{\text{particle bias (I)}}  + \underbrace{\| \E(\hat{g}_\theta(S,B) - \hat{g}_\theta(S,T)) \|}_{\text{buffering bias (II)}} + \underbrace{\| \E\hat{g}_\theta(S,T) - g_\theta \|}_{\text{subsequence bias (III)}}
 \, ,
\end{align}
where expectations are over the random subsequence $\mathcal{S}$ and particles.
Each term is bounded separately (recalling that $\gamma = \max_t \Pr(t \in \SUBSEQ)^{-1}$)
\begin{enumerate}[(I)]
\item \emph{Particle bias}: the particle filter bias is $\mathcal{O}(\gamma \frac{S+2B}{N})$ (see Eq. 3.15 of~\citet{kantas2015particle}).
\item \emph{Buffering bias}: from~\citet{aicher2019stochastic}, we know there exists a finite constant  $C_1 < \infty$ that is independent of $T,S,B,N$, such that
\begin{equation}
\E\| \hat{g}(S,B) - \hat{g}(S,T) \| \leq  \gamma \cdot C_1 \cdot (L_\theta)^{B} \enspace.
\end{equation}
Thus, the buffering bias can be upper bounded using Jensen's inequality
\begin{align}
\nonumber
\| \E(\hat{g}(S,B) - \hat{g}(S,T)) \| &\leq \E\| \hat{g}(S,B) - \hat{g}(S,T) \| \\
&\leq \gamma \cdot C_1 \cdot (L_\theta)^{B}\enspace.
\end{align}

\item \emph{Subsequence bias}: For this term the randomness is only with respect to the choice of subsampler. By our our decomposition, the subsequence bias is zero, $\E\hat{g}_\theta(S,T) = g_\theta$, as the bias due to using a finite buffer is accounted for in (II).
\end{enumerate}
Applying these bounds gives us the bias bound
\begin{equation}
\|\E g^{\PF}_\theta(S,B,N) - g_\theta \| \leq \gamma \cdot \left[C_1 (L_\theta)^{B} + \mathcal{O}\left(\frac{S + 2B}{N}\right)\right] \, .
\end{equation}

For the MSE bound, \eqref{eq:grad_mse_bound}, we again apply the triangle inequality and recall that $2XY \leq X^2 + Y^2$ implies $(X+Y+Z)^2 \leq 3 (X^2 + Y^2 + Z^2)$ to decompose the error into three terms
\begin{align}
\nonumber
\E\| g^{\PF}_\theta(S, &B, N) - g_\theta \|^2 \leq \\
\nonumber
&3\underbrace{\E\| g^{\PF}_\theta(S,B,N) - \hat{g}_\theta(S,B) \|^2}_{\text{particle MSE (I)}} \\
\nonumber
&+\,3\underbrace{\E\| \hat{g}_\theta(S,B) - \hat{g}_\theta(S,T) \|^2}_{\text{buffering MSE (II)}} \\
&+\,3\underbrace{\E\| \hat{g}_\theta(S,T) - g_\theta \|^2}_{\text{subsequence MSE (III)}}
\label{eq:MSE_decompose}
 \, ,
\end{align}
where expectations are over the random subsequence $\mathcal{S}$ and particles.
Again, each term is bounded separately,
\begin{enumerate}[(I)]
\item \emph{Particle MSE}: the particle filter MSE bound is $\mathcal{O}(\gamma^2 \frac{(S+2B)^2}{N})$ (see Eq. 3.15 of~\citet{kantas2015particle}).
\item \emph{Buffering MSE}: from~\citet{aicher2019stochastic}, the buffering MSE is bounded
\begin{equation}
\E\| \hat{g}_\theta(S,B) - \hat{g}_\theta(S,T)\|^2 \leq \gamma^2 \cdot C_1^2 \cdot (L_\theta)^{2B} \enspace.
\end{equation}
\item \emph{Subsequence MSE}: from Lemma~\ref{lemma:suberror}, there exists a constant $C_2 < \infty$ independent of $T,S,B,N$ such that
\begin{equation}
\E\| \hat{g}_\theta(S,T) - g_\theta \|^2 \leq \gamma^2 \cdot C_2 \cdot S \enspace.
\end{equation}
\end{enumerate}
Combining these bounds gives us the MSE bound
\begin{align}
\E\| g^{\PF}_\theta(&S,B,N) - g_\theta \|^2 \leq 3\gamma^2 \cdot \left[C_1^2 (L_\theta)^{2B} + C_2 S + \mathcal{O}\left(\frac{(S + 2B)^2}{N}\right)\right] \enspace.
\end{align}
\end{proof}

\subsubsection{Stochastic Subsequence MSE}
\label{supp-subsub:suberror}
For the proof of Theorem~\ref{thm:gradient_bias_mse}, we bound the MSE between the full gradient $g_\theta$ and the unbiased stochastic gradient estimate $\hat{g}_\theta(S,T)$, specifically for the case of randomly sampling a \emph{contiguous} subsequence $\SUBSEQ$.
Because $\hat{g}_\theta(S,T)$ is unbiased for $g_\theta$, this reduces to calculating the variance of $\hat{g}_\theta(S,T)$ with respect to the sampling distribution of the subsequence $\SUBSEQ$.

Let $f_t$ denote the $t$-th gradient term in Fisher's identity
\begin{equation}
f_t = \E_{X_{1:T} | y_{1:T}, \theta}[\grad \log p(y_t, X_t \, | \, X_{t-1}, \theta)] \enspace.
\end{equation}
Therefore
\begin{equation*}
g_\theta = \sum_{t = 1}^T f_t 
\enspace \text{ and } \enspace
\hat{g}_\theta(S,T) = \sum_{t \in \SUBSEQ} \Pr(t \in \SUBSEQ)^{-1} \cdot f_t \enspace.
\end{equation*}

We now present the lemma that bounds the variance of $\hat{g}_\theta(S,T)$, under the assumption that the autocorrelation of $f_t$ decays geometrically $|\Corr(f_t, f_{t+s})| \leq \rho^s$.
\begin{lemma}
\label{lemma:suberror}
If for all $t$, the variance of $f_t$ is bounded and the autocorrelation of $f_t$ is geometrically bounded,
then there exists a constant $C_2 < \infty$ (not dependent on $T,S,B$) such that
\begin{equation}
\Var(\hat{g}_\theta(S,T)) \leq \gamma^2 \cdot C_2 \cdot S \enspace.
\end{equation}
\end{lemma}
The assumption that the autocorrelation of $f_t$ decays geometrically is reasonable when both the observations $Y_{1:T}$ and the posterior latent states $X_{1:T} | Y_{1:T}$ are \emph{ergodic} (i.e. exhibit an exponential forgetting property common for most finite dimensional SSMs~\cite{chan1998state,cappe2005inference}).

We now present the proof.
\begin{proof}[Proof of Lemma~\ref{lemma:suberror}]
Let $V < \infty$ be a bound on the variance of $f_t$ for all $t$ (i.e. $\Var(f_t) \leq V$).
Let $\rho \in [0,1)$ be a bound on the geometric decay of the autocorrelation of $f_t$. Then we have $|\Corr(f_t, f_{t+s})| \leq \rho^s$ for all $t$ and $s \in \mathbb{N}$,
Together these bounds imply a bound on the covariance between any $f_t$ and $f_{t+s}$
\begin{align}
\CoV(f_t, f_{t+s})
\leq | \Corr(f_t, f_{t+s}) | \cdot \sqrt{\Var(f_t) \cdot \Var(f_{t+s})}
\leq V \rho^s \enspace.
\end{align}

Then we have
\begin{align}
\nonumber
\Var(\hat{g}_\theta(&S,T)) \leq \gamma^2 \cdot \Var\left[ \sum_{t \in \SUBSEQ} f_t \right] \\
\nonumber
&= \gamma^2 \cdot \left[\sum_{t \in \SUBSEQ} \Var(f_{t}) +  \sum_{t \neq t' \in \SUBSEQ}\CoV(f_{t}, f_{t'}) \right] \\
\nonumber
&\leq \gamma^2 \cdot \left[S \cdot V + \sum_{s = 1}^{S-1} 2 (S-s) \cdot V \rho^s\right] \\
\nonumber
&= \gamma^2 \cdot S \cdot \left[V + 2V \sum_{s = 1}^{S-1} (1 - s/S) \cdot \rho^s \right] \\
\nonumber
&\leq \gamma^2 \cdot S \cdot \left[ 2V\sum_{s = 0}^{S-1} \rho^s \right] \\
&\leq \gamma^2 \cdot S \cdot 2V/(1-\rho) \enspace.
\end{align}
As $S \geq 1$, if $C_2 = 2V/(1-\rho)$, we have
\begin{equation}
\E\|\hat{g}_\theta(S,T) - g_\theta \|^2 = \Var(\hat{g}_\theta(S,T)) \leq \gamma^2 \cdot S \cdot C_2 \enspace.
\end{equation}

\end{proof}

\subsection{Proof of Theorem~\ref{thm:lipschitz_bound_logconcave}}
\label{supp-subsub:lipschitzproof}
Theorem~\ref{thm:lipschitz_bound_logconcave} states that if the prior distribution for $x_0$, the transition distribution $p(x_t \, | \, x_{t-1}, \theta)$ and the emission distribution $p(y_t \, | \, x_t)$ are log-concave, then we can bound the Lipschitz constant of $\forwardkernel_t$ in terms of $\forwardkernel^{(0)}_t$ and $\forwardkernel^{(1)}_t$. 

We first briefly review Wasserstein distance, random mappings, and Lipschitz constants of kernels~\cite{villani2008optimal, aicher2019stochastic}.
Then we review \emph{Caffarelli's log-concave perturbation theorem}, the main tool we use in our proof.
Finally, we present the proof in Section~\ref{supp-subsub:proof}.

\subsubsection{Wasserstein Distance and Random Maps}
The $p$-\emph{Wasserstein distance} with respect to Euclidean distance is
\begin{equation}
\mathcal{W}_p(\gamma, \widetilde\gamma) := \left[\inf_\xi \int \| x - \widetilde{x} \|_2^p \, d\xi(x, \widetilde{x}) \right]^{1/p} \enspace,
\end{equation}
where $\xi$ is a joint measure or \emph{coupling} over $(x, \widetilde{x})$ with marginals
$\int_{\widetilde{x}} d\xi(x, \widetilde{x}) = d\gamma(x) $ and 
$\int_{x} d\xi(x, \widetilde{x}) = d\widetilde\gamma(x) $.

To bound the Wasserstein distance, we first must introduce the concept of a \emph{random mapping} associated with a transition kernel.

Let $\Psi: \mathcal{U} \rightarrow \mathcal{V}$ be a transition kernel for random variables $u$ and $v$, then for any measure $\mu(u)$ over $\mathcal{U}$, we define the induced measure $(\mu \Psi)(v)$ over $\mathcal{V}$ as $(\mu \Psi)(v) = \int \Psi(u, v) \mu(du)$.

A \emph{random mapping} $\psi$ is a random function that maps $\mathcal{U}$ to $\mathcal{V}$ such that if $u \sim \mu$ then $\psi(u) \sim \mu \Psi$.
For example, if $\Psi(u, v) = \mathcal{N}(v \, | \, u, 1)$, then a random mapping for $\Psi$ is the identity function plus Gaussian noise $\psi(u) = u + \epsilon$, where $\epsilon \sim \mathcal{N}(0,1)$.
If $\psi$ is deterministic $(\mu \Psi)(v)$ is the \emph{push-forward} measure of $\mu$ through the mapping $\psi$; otherwise it is the average (or marginal) over $\psi$ of push-forward measures~\cite{villani2008optimal}.

We say the kernel has Lipschitz constant $L$ with respect to Euclidean distance if
\begin{equation}
\| \Psi \|_{Lip} = L \,\Leftrightarrow\, \sup_{u, u'}\left\{\frac{\E_{\psi}[\|\psi(u) - \psi(u') \|_2]}{\| u - u' \|_2} \right\} \leq L \enspace.
\end{equation}
Note that $L$ is an upper-bound on the \emph{expected value} of Lipschitz constants for random instances of $\psi$.

These definitions are useful for proving bounds in Wasserstein distance.
For example, we can show the kernel $\Psi$ induces a contraction in $p$-Wasserstein distance if $\| \Psi \|_{Lip} < 1$.
That is $\mathcal{W}_p(\mu \Psi, \tilde{\mu} \Psi) \leq \| \Psi \|_{Lip} \cdot \mathcal{W}_p(\mu, \tilde{\mu})$
\begin{align}
\nonumber
\mathcal{W}_p(\mu \Psi, \tilde\mu \Psi)^p
&=\inf_{\xi(\mu \Psi, \tilde{\mu} \Psi)} \int \| v - \tilde{v}\|_2^p \, d\xi(v, \tilde{v}) \\
\nonumber
&\leq \inf_{\xi(\mu, \tilde{\mu})}  \int \| \psi(u) - \psi(\tilde{u}) \|_2^p \, d\xi(u, \tilde{u}) d\mu(\psi) \\
\nonumber
&\leq \inf_{\xi(\mu, \tilde{\mu})}  \int \|\Psi\|_{Lip}^p \cdot \| u - \tilde{u} \|_2^p \, d\xi(u, \tilde{u}) \\
&=\|\Psi\|_{Lip}^p \cdot \mathcal{W}_p(\mu, \tilde{\mu})^p\enspace,
\label{eq:wasserstein_random_map}
\end{align}
where in the second line we replace $v, \tilde{v}$ with a random mapping $\psi$ that has measure $d\mu(\psi)$ and in the third line we bound $\psi$ by its Lipschitz constant $L$ and integrate $\int d\mu(\psi) = 1$.

\subsubsection{Caffarelli's Perturbation Theorem}
Caffarelli's log-concave perturbation theorem allows us to connect Lipschitz constants between kernels that are log-concave perturbations of one another.

\begin{theorem}[Caffarelli's]
\label{thm:caffarelli}
Suppose $\gamma(x)$ is a log-concave measure for $x$ and $\ell(x)$ is a log-concave function such that $\gamma'(x) = \ell(x) \gamma(x)$ is a probability measure over $x$.
Then there exists a $1$-Lipschitz mapping $T: \mathcal{X} \rightarrow \mathcal{X}$ such that
if $x \sim \gamma(x)$ then $T(x) \sim \gamma'(x)$.
\end{theorem}
We can think of $\gamma(x)$ as a prior distribution $p(x)$,
$\ell(x)$ as a normalized conditional likelihood $p(y | x)/p(y)$ and
$\gamma'(x)$ as the posterior $p(x | y)$.
As $\ell(x)$ is log-concave, we call $\gamma'(x)$ a \emph{log-concave perturbation} of $\gamma$.

The original version of Caffarelli's log-concave perturbation theorem~\cite{colombo2017lipschitz, saumard2014log} requires the prior $\gamma(x)$ to be \emph{strongly} log-concave (e.g. a Gaussian) to show that the mapping $T$ is a \emph{strict} contraction $\|T\|_{Lip} < 1$; however this weaker version, Theorem~\ref{thm:caffarelli} of~\citet{villani2008optimal}, is sufficient for our purposes.

\subsubsection{Proof of Theorem~\ref{thm:lipschitz_bound_logconcave}}
\label{supp-subsub:proof}
Using Theorem~\ref{thm:caffarelli}, we can now prove Theorem~\ref{thm:lipschitz_bound_logconcave}.
\begin{proof}[Proof of Theorem~\ref{thm:lipschitz_bound_logconcave}]
Let $\forwardmap_t, \forwardmap^{(0)}_t, \forwardmap^{(1)}_t$ be random mappings associated respectively with forward kernels $\forwardkernel_t, \forwardkernel^{(0)}_t, \forwardkernel^{(1)}_t$.
Because the transition and emission distributions are log-concave and log-concavity is preserved under product and marginalization~\cite{saumard2014log},
$\forwardkernel_t$, $\forwardkernel^{(0)}_t$, $\forwardkernel^{(1)}$ are log-concave and
$p(y_{t\geq} \, | \, x_t)$ and $p(y_{>t} \, | \, x_t)$ are also log-concave (where $y_{t \geq} = \{y_t, \ldots, y_T\}$ and $y_{>t} = \{y_{t+1}, \ldots, y_T\}$).

Since $p(y_{\geq t} \, | \, x_t)$ is log-concave, we can write $\forwardkernel_t$ as a log-concave perturbation of $\forwardkernel^{(0)}_t$,
\begin{align}
\nonumber
\forwardkernel_t = p(x_t \, | \, x_{t-1}, y_{t:T} , \theta) &\propto p(y_{\geq t} \, | \, x_t) p(x_t \, | \, x_{t-1}, \theta)  \\
&= p(y_{\geq t} \, | \, x_t) \cdot \forwardkernel^{(0)}_t \enspace.
\end{align}
Therefore, there exists $T^{(0)}_t$ with $\| T^{(0)}_t \|_{Lip} \leq 1$ such that
$\forwardmap_t = (T^{(0)}_t \circ \forwardmap^{(0)}_t)$.
Thus,
\begin{equation}
\|\forwardkernel_t \|_{Lip} = \|T^{(0)}_t \|_{Lip} \cdot \|\forwardkernel^{(0)}_t \|_{Lip} \leq \|\forwardkernel^{(0)}_t \|_{Lip} \enspace .
\end{equation}

Similarly, we can write $\forwardkernel_t$ as a log-concave perturbation of $\forwardkernel^{(1)}_t$ using $p(y_{> t} \, | \, x_t)$, thus $\|\forwardkernel_t \|_{Lip} \leq \|\forwardkernel^{(1)}_t \|_{Lip}$.
\begin{align}
\nonumber
\forwardkernel_t = p(x_t \, | \, x_{t-1}, y_{t:T} , \theta) &\propto p(y_{>t} \, | \, x_t) p(x_t \, | \, y_t, x_{t-1}, \theta)  \\
&= p(y_{> t} \, | \, x_t) \cdot \forwardkernel^{(1)}_t \enspace.
\end{align}

\end{proof}

Note the assumptions for equivalent results in the backward smoothers $\backwardkernel_t$ are identical.
Log-concavity in $p(x_t \, | \, x_{t+1}, \theta)$ is implied from both $p(x_t \, | \, x_{t-1}, \theta)$ and the prior $p(x_t)$ being log-concave.

\subsection{Bounds for Specific Models}
\label{supp-sub:model_bounds}
We now provide specific bounds for the buffering error for models we consider in Section~\ref{sec:exp} (LGSSM and SVM) using Theorem~\ref{thm:lipschitz_bound_logconcave}.

For both the LGSSM and SVM, we assume the prior $\nu(x_0 | \theta) = \mathcal{N}(0, \sigma^2/(1-\phi^2))$.
Then the latent state transitions are 
\begin{align*}
p(x_t \, | \, x_{t-1}, \theta) &= \mathcal{N}(x_t \, | \, \phi x_{t-1}, \sigma^2) \\
p(x_t \, | \, x_{t+1}, \theta) &= \mathcal{N}(x_t\, | \, \phi x_{t+1} , \sigma^2) \enspace,
\end{align*}
which are both Gaussian and therefore log-concave in $x$. \\ 

Similarly, the emissions for the LGSSM and SVM are also log-concave in $x$:\\
For the LGSSM, 
\begin{equation*}
p(y_t \, | \, x_t, \theta) \propto \exp\left(-\tfrac{(y_t - x_t)^2}{2\sigma^2}\right) \enspace,
\end{equation*}
which is log-concave.\\
For the SVM,
\begin{equation*}
p(y_t \, | \, x_t, \theta) \propto \exp\left(-\tfrac{y_t^2}{2\sigma^2} \cdot e^{-x_t} - \tfrac{x_t}{2}\right) \enspace,
\end{equation*}
which is log-concave as $e^{-x}$ is convex.

\subsubsection{Contraction Bound for LGSSM}
We assume the prior $\nu(x_0 | \theta) = \mathcal{N}(0, \sigma^2/(1-\phi^2))$.
For the LGSSM, the filtered kernels are
\begin{align}
\nonumber
\forwardkernel_t^{(1)}(x_t \, | \, x_{t-1}) &= p(x_t \, | \, x_{t-1}, y_t, \theta) \\
\nonumber
&\propto \mathcal{N}(x_t | \phi x_{t-1}, \sigma^2) \cdot \mathcal{N}(y_t | x_t, \tau^2), \\
\nonumber
\backwardkernel_t^{(1)}(x_t \, | \, x_{t+1}) &= p(x_t \, | \, x_{t+1}, y_t, \theta) \\
&\propto \mathcal{N}(x_t | \phi x_{t+1}, \sigma^2) \cdot \mathcal{N}(y_t | x_t, \tau^2).
\end{align}
Therefore,
\begin{align}
\nonumber
\forwardkernel_t^{(1)}(x_t \, | \, x_{t-1}) &= \mathcal{N}\left(x_t \, \Big| \, \frac{\sigma^2 y_t + \phi \tau^2 x_{t-1}}{\sigma^2 + \tau^2} \, , \, \frac{\sigma^2 \tau^2}{\sigma^2 + \tau^2}\right), \\
\backwardkernel_t^{(1)}(x_t \, | \, x_{t+1}) &= \mathcal{N}\left(x_t \, \Big| \, \frac{\sigma^2 y_t + \phi \tau^2 x_{t+1}}{\sigma^2 + \tau^2} \, , \, \frac{\sigma^2 \tau^2}{\sigma^2 + \tau^2}\right).
\end{align}
The associated random mapping are,
\begin{align}
\nonumber
\forwardmap_t^{(1)}(x_t \, | \, x_{t-1}) &= \frac{\sigma^2 y_t}{\sigma^2 + \tau^2} + \frac{\phi \tau^2}{\sigma^2 + \tau^2} \cdot  x_{t-1} + \vec{z}_t \, , \\
\backwardmap_t^{(1)}(x_t \, | \, x_{t+1}) &= \frac{\sigma^2 y_t}{\sigma^2 + \tau^2} + \frac{\phi \tau^2}{\sigma^2 + \tau^2} \cdot  x_{t+1} + \cev{z}_t \, ,
\end{align}
where $\vec{z}_t$ and $\cev{z}_t$ are $\mathcal{N}\left(0 \, , \, \tfrac{\sigma^2 \tau^2}{\sigma^2 + \tau^2}\right)$ random variables.

Since these maps are linear, we have $\| \forwardkernel_t^{(1)} \|_{Lip} = \| \backwardkernel_t^{(1)}\|_{Lip} = |\phi| \cdot \tfrac{\tau^2}{\sigma^2 + \tau^2}$.
Applying Theorem~\ref{thm:lipschitz_bound_logconcave}, we obtain
\begin{equation}
L_\theta \leq \max_t \{ \| \forwardkernel_t^{(1)}\|, \| \backwardkernel_t^{(1)} \| \} =  |\phi| \cdot (1 + \sigma^2/\tau^2)^{-1}.
\end{equation}
Therefore $L_\theta < 1$ whenever $|\phi| < 1 + \sigma^2/\tau^2$.

\subsubsection{Contraction Bound for SVM}
We assume the prior $\nu(x_0 \theta) = \mathcal{N}(0, \sigma^2/(1-\phi^2))$. For the SVM, the prior kernels are,
\begin{align}
\nonumber
\forwardkernel_t^{(0)}(x_t \, | \, x_{t-1}) &= p(x_t \, | \, x_{t-1}, \theta) \propto \mathcal{N}(x_t | \phi x_{t-1}, \sigma^2),\\
\backwardkernel_t^{(0)}(x_t \, | \, x_{t+1}) &= p(x_t \, | \, x_{t+1}, \theta) \propto \mathcal{N}(x_t | \phi x_{t+1}, \sigma^2).
\end{align}
The associated random mapping are
\begin{align}
\nonumber
\forwardmap_t^{(0)}(x_t \, | \, x_{t-1}) &= \phi \cdot x_{t-1} + \mathcal{N}\left(0 \, , \, \sigma^2\right), \\
\backwardmap_t^{(0)}(x_t \, | \, x_{t+1}) &= \phi \cdot x_{t+1} + \mathcal{N}\left(0 \, , \, \sigma^2\right).
\end{align}
Applying Theorem~\ref{thm:lipschitz_bound_logconcave}, we obtain $L_\theta \leq |\phi|$.

\section{Model Details}
\label{supp-sec:models}
\subsection{LGSSM}
The LGSSM in this paper is given by
\begin{align}
\nonumber
X_t \, | \, (X_{t-1} = x_{t-1}), \theta &\sim \mathcal{N}(x_t \, |\, \phi x_{t-1} \, ,\, \sigma^2), \\
Y_t \, | \, (X_t = x_t), \theta &\sim \mathcal{N}(y_t \, |\, x_t \, ,\, \tau^2), 
\end{align}
with parameters $\theta = (\phi, \sigma, \tau)$.  \\

When applying the particle filter, Algorithm~1, to the LGSSM, we consider two proposal densities $q(\cdot | \cdot)$:
\begin{itemize}
\item The prior (transition) kernel
\begin{equation}
X_t \, | \, (X_{t-1} = x_{t-1}), \theta \sim \mathcal{N}(x_t  \, | \, \phi x_{t-1}, \sigma^2),
\end{equation}
where the weight update, (3), is
\begin{equation}
w_t^{(i)} \propto \frac{1}{\sqrt{2\pi \tau^2}}\exp\left(\frac{-(y_t - x^{(i)}_t)^2}{2\tau^2}\right).
\end{equation} 
\item The `optimal instrumental kernel' 
\begin{align}
\nonumber
X_t \, |& \, (X_{t-1} = x_{t-1}, Y_t = y_t), \theta \\
&\sim \mathcal{N}\left(x_t \, \Big| \, \frac{\tau^2 \phi x_{t-1} + \sigma^2 y_t}{\sigma^2 + \tau^2}, \frac{\sigma^2 \tau^2}{\sigma^2 + \tau^2}\right),
\end{align}
where the weight update, (3), is
\begin{equation}
w_t^{(i)} \propto \frac{1}{\sqrt{2\pi (\sigma^2 + \tau^2)}}\exp\left(\frac{-(y_t - \phi x^{(a_i)}_{t-1})^2}{2(\sigma^2+\tau^2)}\right).
\end{equation}
\end{itemize}
In our experiments with the LGSSM, we use the optimal instrumental kernel.

For this model, the (elementwise) complete data loglikelihood is
\begin{align}
\nonumber
\log p(y_t,&\, x_t \, | \, x_{t-1}, \theta) = -\log(2\pi) - \log(\sigma) \\
&-\frac{(x_t - \phi x_{t-1})^2}{2\sigma^2} - \log(\tau) - \frac{(y_t - x_t)^2}{2\tau^2}.
\end{align}

The gradient of the complete data loglikelihood is then,
\begin{align}
\nonumber
\grad_\phi \log p(y_t, x_t \, | \, x_{t-1}, \theta) &= \frac{(x_t - \phi x_{t-1}) \cdot x_{t-1}}{\sigma^2}, \\
\nonumber
\grad_\sigma \log p(y_t, x_t \, | \, x_{t-1}, \theta) &= \frac{(x_t - \phi x_{t-1})^2 - \sigma^2}{\sigma^3}, \\
\grad_\tau \log p(y_t, x_t \, | \, x_{t-1}, \theta) &= \frac{(y_t - x_t)^2 - \tau^2}{\tau^3}. 
\end{align}

We reparametrize the gradients with $\sigma^{-1}$ and $\tau^{-1}$ to obtain,
\begin{align}
\nonumber
\grad_{\sigma^{-1}} \log p(y_t, x_t \, | \, x_{t-1}, \theta) &= \frac{\sigma^2 - (x_t - \phi x_{t-1})^2}{\sigma}, \\
\grad_{\tau^{-1}} \log p(y_t, x_t \, | \, x_{t-1}, \theta) &= \frac{\tau^2 - (y_t - x_t)^2}{\tau}.
\end{align} \\
To complete the SGMCMC scheme, the prior distributions of the parameters $\theta$ are given as follows:
\begin{align}
\nonumber
\phi &\sim \mathcal{N}(0, 100 \cdot \sigma^2) \\
\nonumber
\sigma^{-1} &\sim Gamma(1+100, (1+100)^{-1}) \\ 
\tau^{-1} &\sim Gamma(1+100, (1+100)^{-1}) \enspace .
\end{align}
The initial parameter values for synthetic experiments were drawn from:
\begin{align}
\nonumber
\phi &\sim \mathcal{N}(0, 1 \cdot \sigma^2) \\
\nonumber
\sigma^{-1} &\sim Gamma(2, 0.5) \\
\tau^{-1} &\sim Gamma(2, 0.5) \enspace.
\end{align}

\subsection{SVM}
The SVM in this paper is given by,  
\begin{align}
\nonumber
X_t \, | \, (X_{t-1}=x_{t-1}), \theta &\sim \mathcal{N}(x_t \,|\, \phi x_{t-1}\,,\, \sigma^2), \\
Y_t \, | \, (X_t = x_t), \theta &\sim \mathcal{N}(y_t \,|\, 0\,,\, \exp(x_t)\tau^2),
\end{align}
with parameters $\theta = (\phi, \sigma, \tau)$. In this model, the observations, $y_{1:T}$, represent the logarithm
of the daily difference in the exchange rate and $X$ is the unobserved volatility. We assume that the volatility process is stationary (such that $0 < \phi < 1$), where $\phi$ is the persistence
in volatility and $\tau$ is the instantaneous volatility. \\
For the particle filter, we use the prior kernel as the proposal density $q$
\begin{equation}
X_t \, | \, (X_{t-1} = x_{t-1}), \theta \sim \mathcal{N}(x_t \, | \, \phi x_{t-1}, \sigma^2),
\end{equation}
with weight update
\begin{equation}
w_t^{(i)} \propto \frac{1}{\sqrt{2\pi \tau^2}}\exp\left(\frac{-y_t^2}{2\exp(x^{(i)}_t)\tau^2}\right).
\end{equation}
The elementwise complete data loglikelihood is 
\begin{align}
\nonumber
\log p(&\,y_t, x_t \, | \, x_{t-1}, \theta) = -\log(2\pi) - \log(\sigma) - \log(\tau) \\
&-\frac{(x_t - \phi x_{t-1})^2}{2\sigma^2} - 0.5 x_t - \frac{(y_t)^2}{2\exp(x_t)\tau^2} .
\end{align}
The gradient of the complete data loglikelihood is then, 
\begin{align}
\nonumber
\grad_\phi \log p(y_t, x_t \, | \, x_{t-1}, \theta) &= \frac{(x_t - \phi x_{t-1}) \cdot x_{t-1}}{\sigma^2}, \\
\nonumber
\grad_\sigma \log p(y_t, x_t \, | \, x_{t-1}, \theta) &= \frac{(x_t - \phi x_{t-1})^2 - \sigma^2}{\sigma^3}, \\
\grad_\tau \log p(y_t, x_t \, | \, x_{t-1}, \theta) &= \frac{y_t^2/\exp(x_t) - \tau^2}{\tau^3}.
\end{align} 
We parametrize with $\sigma^{-1}$ and $\tau^{-1}$ to obtain, 
\begin{align}
\nonumber
\grad_{\sigma^{-1}} \log p(y_t, x_t \, | \, x_{t-1}, \theta) &= \frac{\sigma^2 - (x_t - \phi x_{t-1})^2}{\sigma}, \\
\grad_{\tau^{-1}} \log p(y_t, x_t \, | \, x_{t-1}, \theta) &= \frac{\tau^2 - y_t^2/\exp(x_t)}{\tau}.
\end{align}
The prior distributions and initializations of the parameters $\theta$ are taken to be the same as in the LGSSM case.

\subsection{GARCH Model}

The GARCH(1,1) model in this paper is given by, 
\begin{align}
\nonumber
&X_t \, | \, (X_{t-1}=x_{t-1}), \sigma_t^2, \theta \sim \mathcal{N}(x_t \,|\, 0,\, \sigma_t^2), \\
\nonumber
&\sigma_t^2(x_{t-1}, \sigma_{t-1}^2, \theta) = \alpha + \beta x_{t-1}^2 + \gamma \sigma_{t-1}^2, \\
&Y_t \, | \, (X_t = x_t), \theta \sim \mathcal{N}(y_t \,|\, x_t\,,\,\tau^2),
\end{align}
where parameters are $\theta = (\log\mu, \logit \phi, \logit \lambda, \tau)$ for $\alpha = \mu (1-\phi)$, $\beta = \phi \lambda$, $\gamma = \phi(1-\lambda)$.
Note that  $\sigma_t^2 = \mu (1-\phi) + \phi (\lambda x_{t-1}^2 + (1-\lambda) \sigma_{t-1}^2)$. 

We consider two proposal densities $q(\cdot | \cdot)$ for the GARCH model:
\begin{itemize}
\item The prior kernel
\begin{align}
\begin{bmatrix} X_t \\ \sigma_t^2 \end{bmatrix}
\, \Big|& \,
\begin{bmatrix} X_{t-1} = x_{t-1} \\ \sigma_{t-1}^2 \end{bmatrix}, \, \theta
\sim  \,
\begin{bmatrix}
\mathcal{N}(x_t \, | \, 0, \alpha + \beta x_{t-1}^2 + \gamma \sigma_{t-1}^2) \\
\delta(\sigma_t^2 \, | \, \alpha + \beta x_{t-1}^2 + \gamma\sigma_{t-1}^2)
\end{bmatrix}.
\end{align}
where the weight update, (3), is
\begin{equation}
w_t^{(i)} \propto \frac{1}{\sqrt{2\pi \tau^2}}\exp\left(\frac{-(y_t - x^{(i)}_t)^2}{2\tau^2}\right).
\end{equation}
\item The optimal instrumental kernel
\begin{align}
\nonumber
\begin{bmatrix} X_t \\ \sigma_t^2 \end{bmatrix}
&\, \Big| \,
\begin{bmatrix} X_{t-1} = x_{t-1} \\ \sigma_{t-1}^2 \end{bmatrix}, \, (Y_t = y_t), \, \theta \\
&\sim  \,
\begin{bmatrix}
\mathcal{N}(x_t \, | \, \sigma_t^2 y_t/(\sigma_t^2 + \tau^2), \sigma_t^2 \tau^2/(\sigma_t^2 + \tau^2)) \\
\delta(\sigma_t^2 \, | \, \alpha + \beta x_{t-1}^2 + \gamma\sigma_{t-1}^2)
\end{bmatrix}.
\end{align}
where the weight update, (3), is
\begin{equation}
w_t^{(i)} \propto \frac{1}{\sqrt{2\pi ((\sigma_t^{(i)})^2 + \tau^2)}}\exp\left(\frac{-y_t^2}{2((\sigma_t^{(i)})^2+\tau^2)}\right).
\end{equation}
\end{itemize}
In our experiments with the GARCH model, we use the optimal instrumental kernel.

The elementwise complete data loglikelihood is 
\begin{align}
\nonumber
\log p(y_t, x_t,&\, \sigma_t^2 \, | \, x_{t-1}, \sigma_{t-1}^2, \theta) = \\
\nonumber
&-\frac{\log(2\pi) +\log(\alpha + \beta x_{t-1}^2 + \gamma \sigma_{t-1}^2)}{2} \\
\nonumber
& -\frac{x_t^2}{2(\alpha+\beta x_{t-1}^2 + \gamma \sigma_{t-1}^2)} \\
&-0.5\log(2\pi) - \log(\tau) - \frac{(y_t - x_t)^2}{2\tau^2}.
\end{align}
Let $\mathcal{L}_t = \log p(y_t, x_t, \sigma^2_t | x_{t-1}, \sigma_{t-1}^2, \theta)$
and set $C_t = \tfrac{x_t^2 - \sigma_t^2}{2 \sigma_t^4}$.
Then the gradient of the complete data log-likelihood $\grad \mathcal{L}_t$ is
\begin{align}
\nonumber
\grad_\tau \mathcal{L}_t  &= \frac{(y_t - x_t)^2 - \tau^2}{\tau^3}, \\
\nonumber
\grad_{\log\mu} \mathcal{L}_t &= 
C_t\cdot (1-\phi) \cdot \mu, \\
\nonumber
\grad_{\logit\phi} \mathcal{L}_t &= 
C_t\cdot (\lambda x_{t-1}^2 + (1-\lambda)\sigma_{t-1}^2 - \mu) \cdot \phi (1-\phi), \\
\grad_{\logit\lambda} \mathcal{L}_t &=  
C_t\cdot (\phi x_{t-1}^2 - \phi \sigma_{t-1}^2) \cdot \lambda (1-\lambda).
\end{align}

The SGMCMC scheme is completed by setting the prior distributions for the parameters as follows:  $(\phi +1)/2 \sim \text{Beta}(10, 1.5)$, $\mu \sim \text{Uniform}(0, 2)$, $(\lambda+1)/2 \sim  \text{Beta}(20, 1.5)$ and
$\tau^2 \sim \mathcal{IG}(2, 0.5)$.

\section{Additional Experiments}
\label{supp-sec:experiments}
We first present the stochastic gradient bias when using other particle filtering methods and when varying the parameters with the LGSSM data.
We then present additional SGLD results on synthetic data for the LGSSM in higher dimensions, the SVM and the GARCH models.
We finally present some additional details for the SGLD experiment on the EUR-US exchange rate data.

\subsection{Gradient Bias with Other Particle Filters}
Figure~\ref{fig:paris-gradbias} compares the stochastic gradient bias of the naive PF with ``PaRIS'' and ``Poyiadjis $N^2$'' on the LGSSM data in Section~\ref{sub:exp_grad}.

From Figure~\ref{fig:paris-gradbias} (top) and (bottom-left), we see that the naive PF (\textcolor{blue}{blue} or solid line) performs similarly to PaRIS (\textcolor{BrickRed}{red} or dashed line) and Poyiadjis $N^2$ (\textcolor{ForestGreen}{green} or dot-dashed line) as $N$ varies.
However, Figure~\ref{fig:paris-gradbias} (bottom-right) shows that the naive PF is about 10 times faster per iteration than PaRIS and Poyiadjis $N^2$.

\begin{figure}[ht]
\begin{center}
\begin{minipage}[c]{.5475\textwidth}
    \centering
        \includegraphics[width=\textwidth, trim=0.08in 0 0 0, clip]{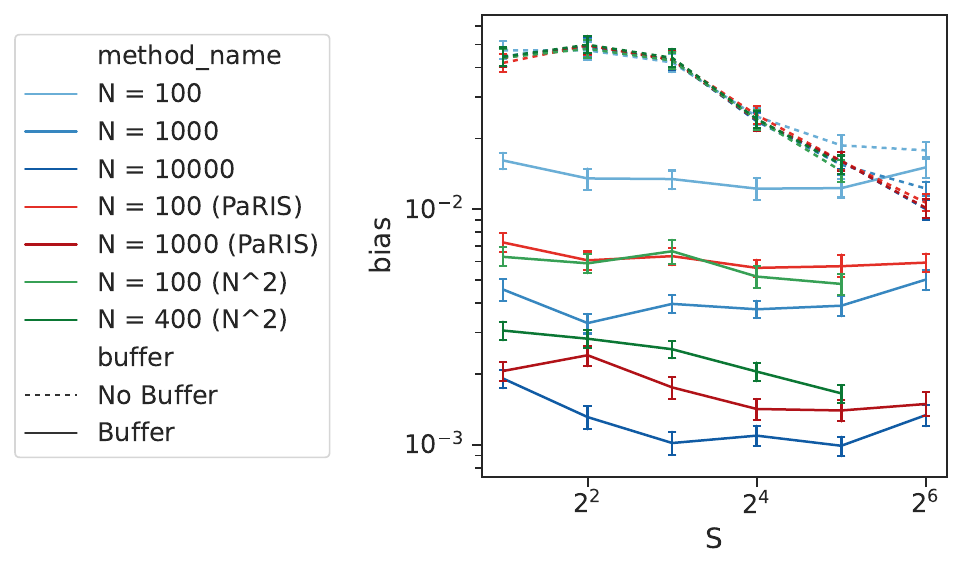}
\end{minipage}
\begin{minipage}[c]{.37\textwidth}
    \centering
        \includegraphics[width=\textwidth, trim=2.25in 0 0 0, clip]{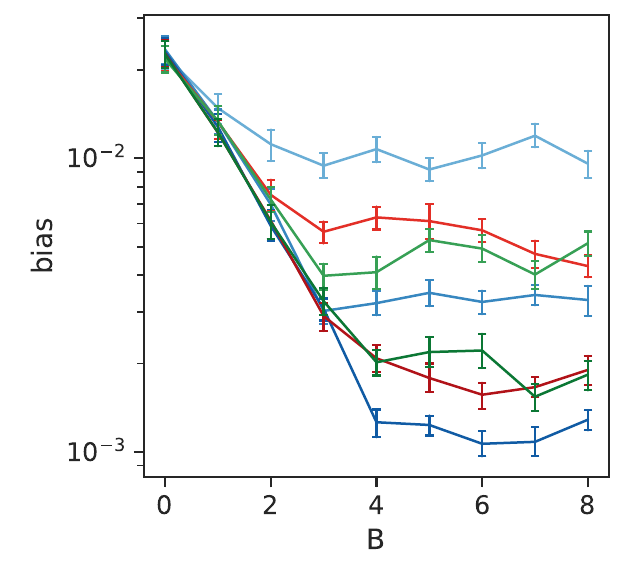}
\end{minipage}
\vspace{0.1em}

\begin{minipage}[c]{.5475\textwidth}
    \centering
        \includegraphics[width=\textwidth, trim=0.09in 0 0 0, clip]{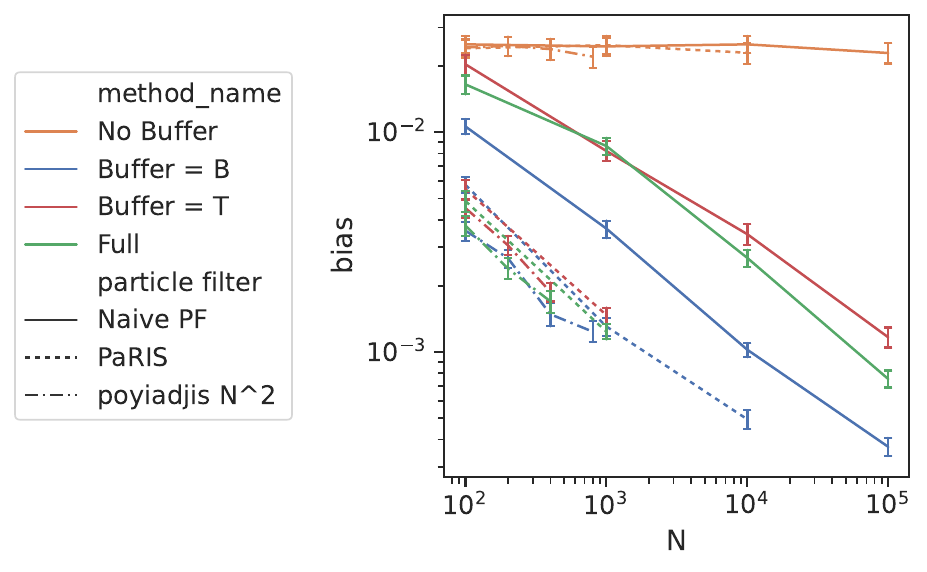}
\end{minipage}
\begin{minipage}[c]{.37\textwidth}
    \centering
        \includegraphics[width=\textwidth, trim=2in 0 0 0, clip]{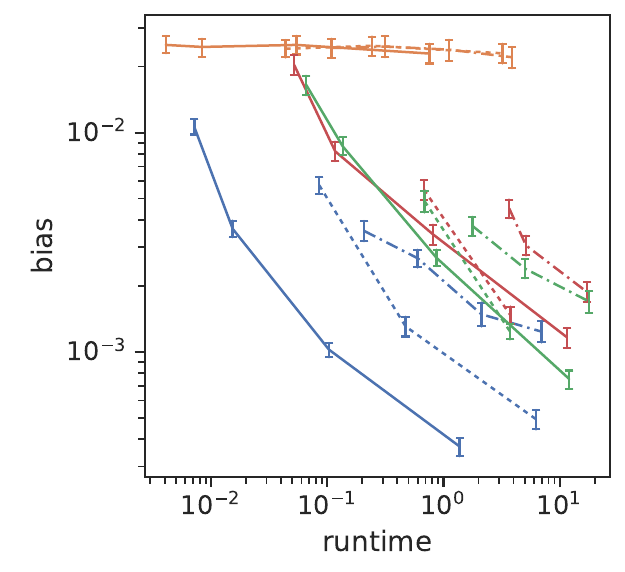}
\end{minipage}
\caption{Stochastic gradient bias varying $B,S,N$ for the naive PF and PaRIS on the LGSSM data.
(Top-left) bias vs $S$, (top-right) bias vs $B$, (bottom-left) bias vs $N$, (bottom-right) bias vs runtime in seconds.
}
\label{fig:paris-gradbias}
\end{center}
\vskip -0.2in
\end{figure}

\subsection{Gradient Bias Varying Parameters}
\begin{figure}[!ht]
\begin{center}
\begin{minipage}[c]{.5475\textwidth}
    \centering
        \includegraphics[width=\textwidth, trim=0.09in 0 0 0, clip]{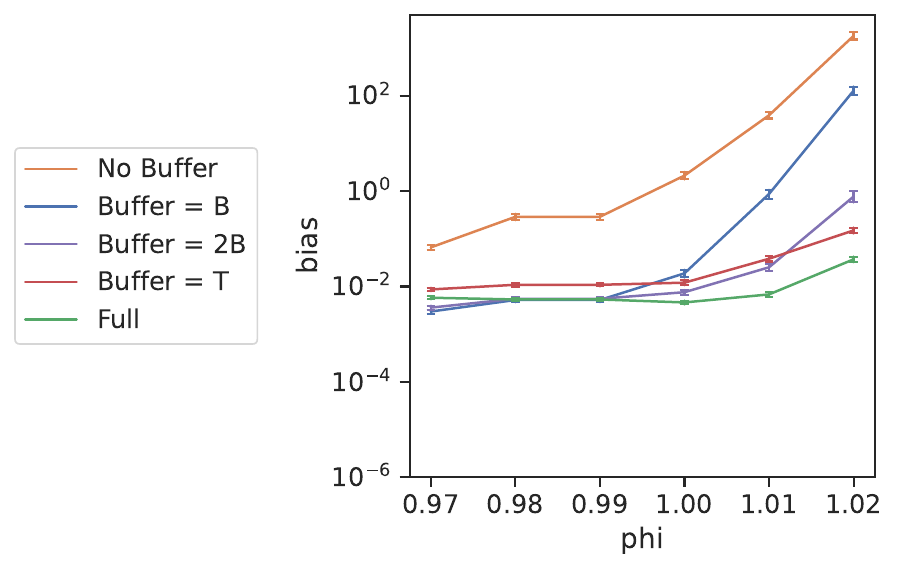}
\end{minipage}
\begin{minipage}[c]{.37\textwidth}
    \centering
        \includegraphics[width=\textwidth, trim=2in 0 0 0, clip]{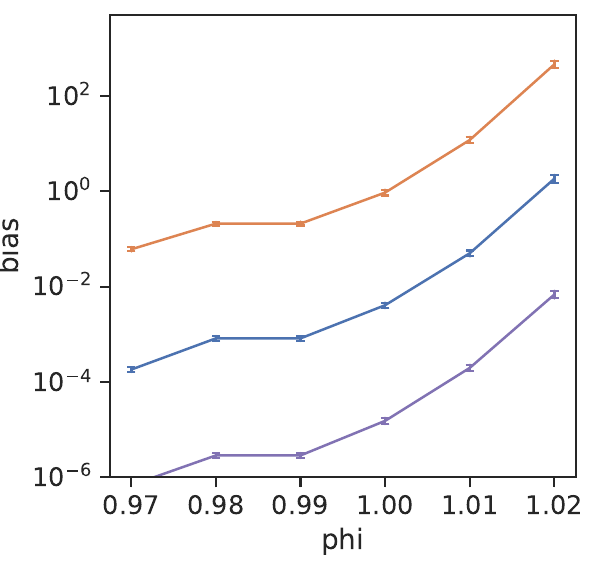}
\end{minipage}
\caption{Stochastic gradient bias varying $\phi$ with $S=16, B=8$ for (left) naive PF $N = 1000$, (right) Kalman filter $N = \infty$.}
\label{fig:vary-params}
\end{center}
\vskip -0.2in
\end{figure}

Figure~\ref{fig:vary-params} compares the stochastic gradient bias for different values of $\phi \in (-0.97, 1.02)$ for the LGSSM experiment in Section~\ref{sub:exp_grad} and shows the trade-off between the buffering error (II) and particle error (III) as $\phi$ (and therefore $L_\theta$) varies.

From Figure~\ref{fig:vary-params} (left) the buffer methods are worse than using the full buffer (\textcolor{BrickRed}{red}) for $\phi > 1.00$ with $B = 8$ (\textcolor{blue}{blue}), and $\phi > 1.01$ for $B = 16$ (\textcolor{RoyalPurple}{purple}). 
This is because the buffering error (II) decays less rapidly with $B$ as $\phi$ increases.

Comparing the naive PF ($N = 1000$) to the Kalman filter, Figure~\ref{fig:vary-params} (left vs right), we see there is a large gap due to particle error (III) as well. 
Therefore, as $\phi$ increases, both $B$ and $N$ need to increase to control bias; otherwise the buffered methods have larger bias than full sequence gradients (\textcolor{ForestGreen}{green})

And again, in all cases, not using a buffer (\textcolor{orange}{orange}) has the largest bias.

\subsection{SGLD on Synthetic Data}
\subsubsection{Additional MSE Figures for LGSSM}
Figure~\ref{fig:lgssm-extra} presents extra MSE plots for the parameters not presented in the main paper.
Tables~\ref{tab:lgssm_1k} and \ref{tab:lgssm_1m} present the full KSD results for each variable.

\begin{figure}[ht]
\vskip 0.1in
\begin{center}
\begin{minipage}[c]{.3\textwidth}
    \centering
        \includegraphics[width=\textwidth]{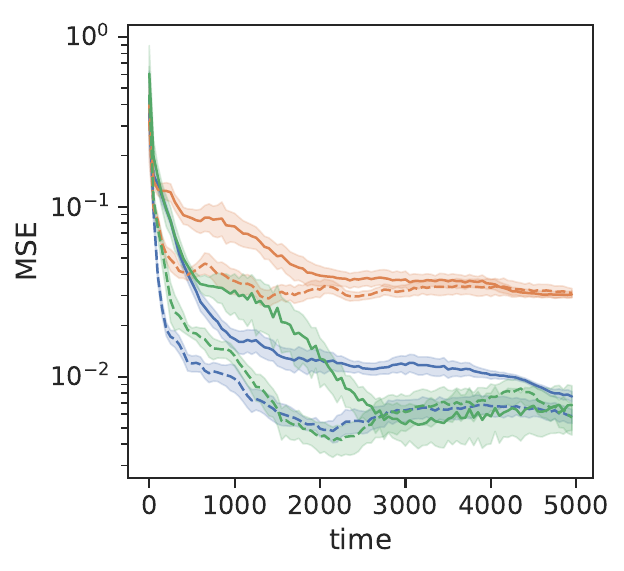}
\end{minipage}
\begin{minipage}[c]{.3\textwidth}
    \centering
        \includegraphics[width=\textwidth]{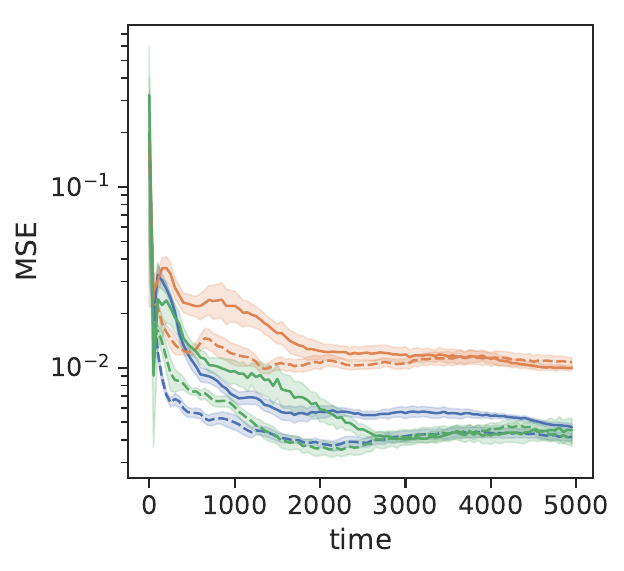}
\end{minipage}
\vspace{0.1em}

\begin{minipage}[c]{.3\textwidth}
    \centering
        \includegraphics[width=\textwidth]{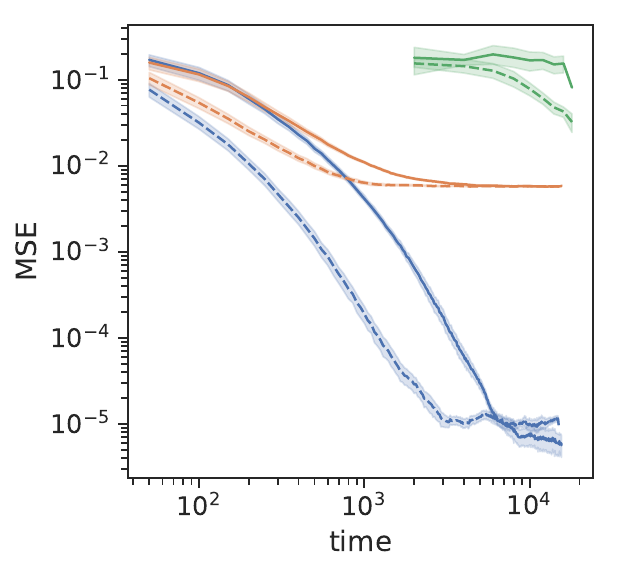}
\end{minipage}
\begin{minipage}[c]{.3\textwidth}
    \centering
        \includegraphics[width=\textwidth]{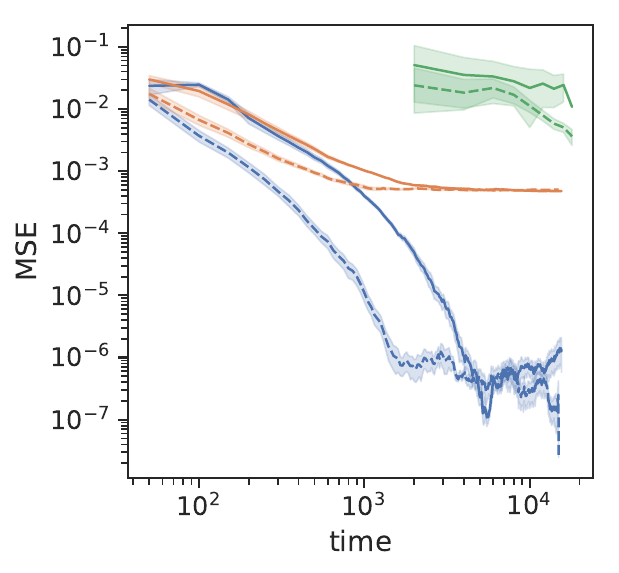}
\end{minipage}
\caption{Additional metrics for SGLD on LGSSM: (left) MSE of $\sigma$, (right) MSE of $\tau$, (top) $T = 10^3$, (bottom) $T = 10^6$}.
\label{fig:lgssm-extra}
\end{center}
\vskip -0.2in
\end{figure}

\begin{table*}[ht]
\caption{KSD results for Synthetic LGSSM with $T = 10^3$.}
\label{tab:lgssm_1k}
\vskip 0.15in
\begin{center}
\begin{small}
\begin{sc}
\begin{tabular}{lllllll}
\toprule
    &     & {} & \multicolumn{4}{l}{$\log_{10}$KSD} \\
$S$ & $B$ & method &          $\phi$ &         $\sigma$ &           $\tau$ &        total \\
\midrule
$10^3$ & --  & Gibbs &  0.09 (0.25) &  -0.02 (0.01) &  -0.16 (0.48) &  \textbf{0.51 (0.13)} \\
    &     & KF &  0.01 (0.57) &   0.07 (0.09) &   0.20 (0.28) &  \textbf{0.64 (0.17) } \\
    &     & PF &  0.38 (0.26) &   0.10 (0.16) &   0.44 (0.19) &  0.85 (0.08) \\
 \midrule
 40 &  0  & KF &  1.53 (0.03) &  -0.08 (0.07) &  -0.04 (0.16) &  1.55 (0.03) \\
    &     & PF &  1.55 (0.03) &  -0.04 (0.13) &   0.10 (0.26) &  1.58 (0.03) \\
 \midrule
 40 &  10 & KF &  0.18 (0.27) &   0.02 (0.07) &   0.04 (0.44) &  \textbf{0.61 (0.21)} \\
    &     & PF &  0.27 (0.46) &   0.09 (0.13) &  -0.11 (0.53) &  \textbf{0.68 (0.25)} \\
\bottomrule
\end{tabular}
\end{sc}
\end{small}
\end{center}
\vskip -0.1in
\end{table*}

\begin{table*}[ht]
\caption{KSD results for Synthetic LGSSM with $T = 10^6$.}
\label{tab:lgssm_1m}
\vskip 0.15in
\begin{center}
\begin{small}
\begin{sc}
\begin{tabular}{lllllll}
\toprule
    &     & {} & \multicolumn{4}{l}{$\log_{10}$KSD} \\
$S$ & $B$ & method &          $\phi$ &         $\sigma$ &           $\tau$ &        total \\
\midrule
$10^6$  & --  & Gibbs &  3.91 (0.80) &  3.43 (1.07) &  3.52 (0.73) &  4.23 (0.74) \\
    &     & KF &  4.51 (0.48) &  4.21 (0.50) &  3.65 (0.55) &  4.85 (0.36) \\
    &     & PF &  4.77 (0.39) &  4.11 (0.57) &  3.55 (0.95) &  4.92 (0.40) \\
\midrule
 40 &  0  & KF &  4.64 (0.14) &  3.25 (0.21) &  2.83 (0.61) &  4.68 (0.11) \\
    &     & PF &  4.64 (0.13) &  3.19 (0.35) &  3.12 (0.45) &  4.68 (0.10) \\
\midrule
 40 &  10 & KF &  3.04 (0.39) &  1.57 (0.50) &  2.68 (0.20) &  \textbf{3.25 (0.29)} \\
    &     & PF &  3.26 (0.17) &  1.70 (0.38) &  2.87 (0.33) &  \textbf{3.43 (0.19)} \\\bottomrule
\end{tabular}
\end{sc}
\end{small}
\end{center}
\vskip -0.1in
\end{table*}

\subsubsection{Higher Dimensional LGSSM}
We generate synthetic LGSSM data for $X_t, Y_t \in \R^d$ using $\phi = 0.9 \cdot \mathbb{I}_d$, $\sigma = 0.7 \cdot \mathbb{I}_d$, and $\tau = \mathbb{I}_d$ for dimensions $d \in \{5, 10\}$ with $T=1000$.
Figure~\ref{fig:lgssm-dim} presents the MSE trace plots for $d = 5$ and for $d=10$.
Table~\ref{tab:lgssm_higher_dim} presents the KSD tables for both $d=5$ and $d=10$.

\begin{figure}[ht]
\begin{center}
\begin{minipage}[c]{.45\textwidth}
    \centering
        \includegraphics[width=\textwidth]{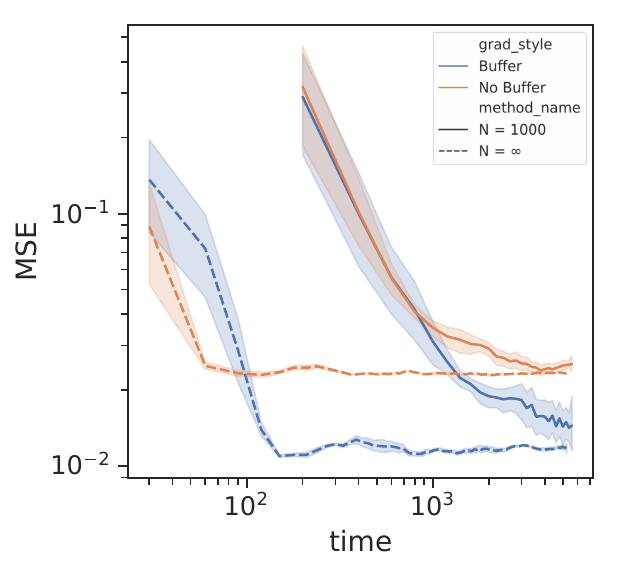}
\end{minipage}
\vspace{10pt}
\begin{minipage}[c]{.45\textwidth}
    \centering
        \includegraphics[width=\textwidth]{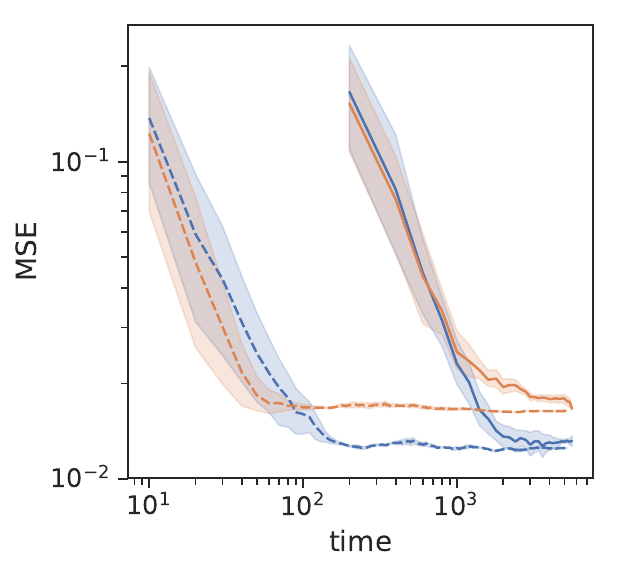}
\end{minipage}
\caption{SGLD Results for LGSSM. MSE of $\phi$ (left) for $X \in \R^5$, (right) $X \in \R^{10}$.}
\label{fig:lgssm-dim}
\end{center}
\vskip -0.2in
\end{figure}


We find that the Kalman filter $N = \infty$ is able to much more rapidly mix compared to the particle filter with $N = 1000$.
This is both due to the increased particle filter variance in higher dimensions and the longer computation required for sampling particles in higher dimensions.
However in both cases, we again see that buffering is necessary to avoid bias.

\begin{table*}[!ht]
\caption{KSD results for Synthetic LGSSM in higher dimensions.}
\label{tab:lgssm_higher_dim}
\vskip 0.15in
\begin{center}
\begin{small}
\begin{sc}
\begin{tabular}{l|l l|llll}
\toprule
  & & &\multicolumn{4}{c}{$\log_{10}$KSD} \\
\midrule
Dim & Grad Est. & N & $\phi$ & $\sigma$ &  $\tau$ & Total \\
\midrule
 5  & No Buffer & 1000 &  1.78 (0.04) &  1.97 (0.26) &  1.44 (0.45) &  2.28 (0.20) \\
    &    & $\infty$ &  1.74 (0.01) &  2.09 (0.02) &  1.64 (0.02) &  2.35 (0.01) \\
    & Buffer & 1000 &  1.18 (0.17) &  1.74 (0.25) &  1.44 (0.03) &  \textbf{2.01 (0.13)} \\
    &    & $\infty$ &  0.84 (0.03) &  1.97 (0.03) &  1.40 (0.05) &  \textbf{2.10 (0.03)} \\
 \midrule
 10 & No Buffer & 1000 &  1.84 (0.01) &  2.40 (0.06) &  2.26 (0.13) &  2.71 (0.06) \\
    &    & $\infty$ &  1.79 (0.01) &  2.13 (0.04) &  2.12 (0.01) &  2.52 (0.02) \\
    & Buffer & 1000 &  1.60 (0.13) &  2.37 (0.04) &  2.20 (0.04) &  2.64 (0.04) \\
    &    & $\infty$ &  1.04 (0.06) &  2.08 (0.04) &  2.07 (0.01) &  \textbf{2.39 (0.02)} \\
\bottomrule
\end{tabular}
\end{sc}
\end{small}
\end{center}
\vskip -0.1in
\end{table*}

\subsubsection{SVM}
Figure~\ref{fig:synth-svm} presents the MSE plots for SGLD on the synthetic SVM data $T = 1000$ and
Table~\ref{tab:synth-svm} presents the KSD for each sampled chain.

We find that buffering performs best (as measured by KSD).
From Figure~\ref{fig:synth-svm} we see that not buffering leads to bias, while the full sequence method is nosier (fewer larger steps) compared to the buffer method.

\begin{figure}[!ht]
\begin{center}
\begin{minipage}[c]{.32\textwidth}
    \centering
        \includegraphics[width=\textwidth]{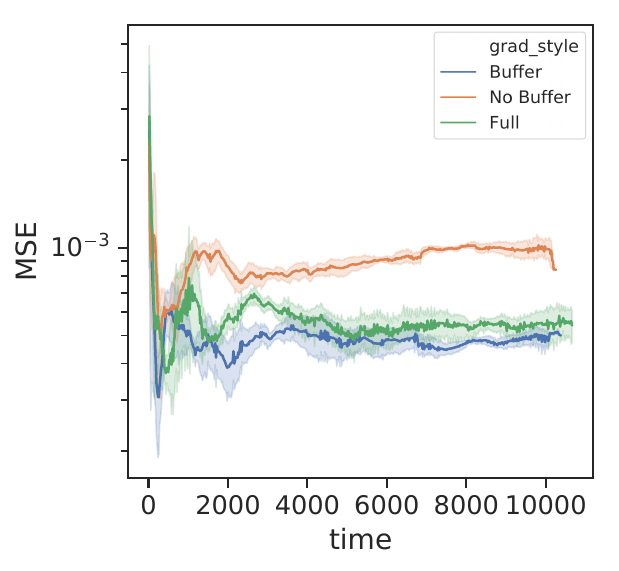}
\end{minipage}
\begin{minipage}[c]{.32\textwidth}
    \centering
        \includegraphics[width=\textwidth]{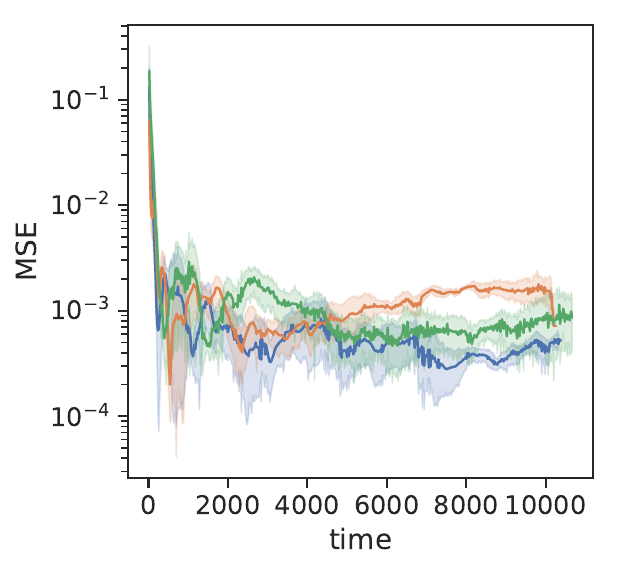}
\end{minipage}
\begin{minipage}[c]{.32\textwidth}
    \centering
        \includegraphics[width=\textwidth]{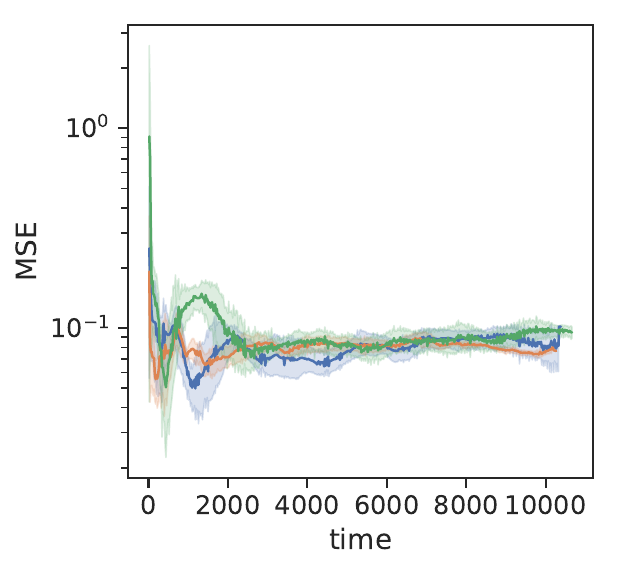}
\end{minipage}
\caption{SGLD results for synthetic SVM data: (left) MSE of $\phi$, (center) MSE of $\sigma$, (right) MSE of $\tau$.}
\label{fig:synth-svm}
\end{center}
\vskip -0.2in
\end{figure}

\subsubsection{GARCH}
Figure~\ref{fig:synth-garch} presents the trace plot metrics for SGLD on the synthetic GARCH data $T = 1000$ and
Table~\ref{tab:synth-garch} presents the KSD for each sampled chain.

We again find that buffering performs best (as measured by KSD).
From Figure~\ref{fig:synth-garch} we see that not buffering leads to bias in sampling $\mu$ and $\lambda$.
The full sequence method encounters high particle error and therefore requires a much longer runtime with a much smaller stepsize to reduce bias.

\begin{figure}[!ht]
\begin{center}
\begin{minipage}[c]{.32\textwidth}
    \centering
        \includegraphics[width=\textwidth]{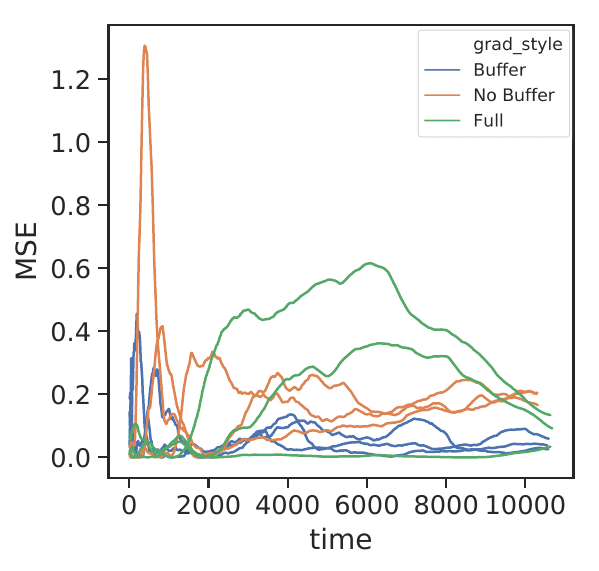}
\end{minipage}
\begin{minipage}[c]{.32\textwidth}
    \centering
        \includegraphics[width=\textwidth]{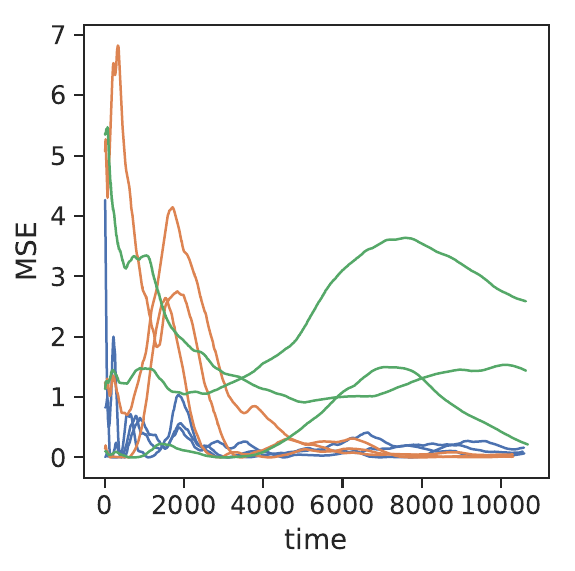}
\end{minipage}
\begin{minipage}[c]{.32\textwidth}
    \centering
        \includegraphics[width=\textwidth]{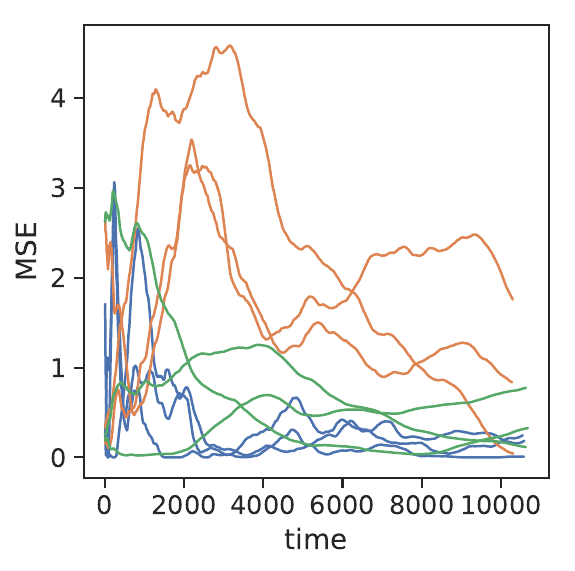}
\end{minipage}
\caption{SGLD results for synthetic GARCH data: (left) MSE of $\log(\mu)$, (center) MSE of $\logit\phi$, (right) MSE of $\logit\lambda$.}
\label{fig:synth-garch}
\end{center}
\vskip -0.2in
\end{figure}

\subsection{SGLD on Exchange Rate}
The EUR-US exchange rate data was pulled from the \url{https://www.finam.ru} website for the time period of November 2017 to October 2018 at the minute resolution. The data is plotted in Figure~\ref{fig:exchange-raw-data}.
\begin{figure}[ht]
\vskip 0.2in
\begin{center}
\begin{minipage}[c]{.6\textwidth}
    \centering
        \includegraphics[width=\textwidth]{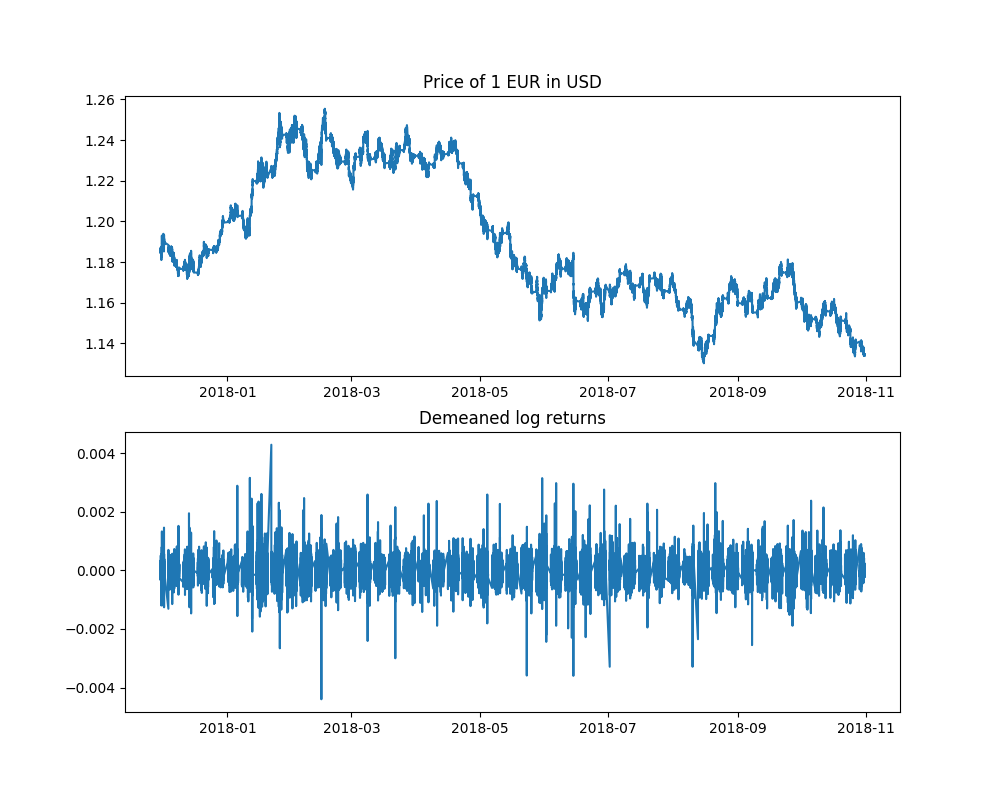}
\end{minipage}
\caption{EUR-US Exchange Rate Data (top) raw data (bottom) demeaned log-returns.}
\label{fig:exchange-raw-data}
\end{center}
\vskip -0.2in
\end{figure} 

The \emph{demeaned log-returns} are calculated by taking the difference of the log-closing price (at each minute) and removing the mean, as done in the \texttt{stochvol} package in R~\cite{stochvolRpackage}
\begin{equation}
\tilde{y}_t = \log(y_t/y_{t-1}) - \frac{1}{T}\sum_{t'} \log(y_{t'}/y_{t'-1}).
\end{equation}

\subsubsection{SVM}
For the SVM, we initialized each chain at $\phi = 0.9$, $\sigma = 1.73$ and $\tau = 0.1$ for all SGLD methods.
The full KSD results are presented in Table~\ref{tab:svm-exchange}.

\subsubsection{GARCH}
For the GARCH model, we initialized each chain at $\log\mu = -0.4$, $\logit\phi = 1.7$, $\logit\lambda = 2.7$ and $\tau = 0.1$  for all SGLD methods. The full KSD results are presented in Table~\ref{tab:garch-exchange}.

\begin{table*}[ht]
\begin{center}
\caption{KSD results for Synthetic SVM.}
\label{tab:synth-svm}
\vskip 0.05in
\begin{small}
\begin{sc}
\begin{tabular}{l|llll}
\toprule
  & \multicolumn{4}{c}{$\log_{10}$KSD} \\
\midrule
Grad Est. & $\phi$ & $\sigma$ &  $\tau$ & Total \\
\midrule
Full &  0.68 (0.28) &   0.38 (0.40) &  0.44 (0.54) &  1.12 (0.22) \\
No Buffer &  1.49 (0.05) &  -0.01 (0.23) &  0.09 (0.35) &  1.53 (0.05) \\
Buffer &  0.35 (0.33) &   0.23 (0.29) &  0.21 (0.40) &  \textbf{0.81 (0.22)} \\
\bottomrule
\end{tabular}
\end{sc}
\end{small}
\vskip 0.2in

\caption{KSD results for Synthetic GARCH.}
\label{tab:synth-garch}
\vskip 0.05in
\begin{small}
\begin{sc}
\begin{tabular}{l|lllll}
\toprule
  & \multicolumn{5}{c}{$\log_{10}$KSD} \\
\midrule
Grad Est. & $\log\mu$ & $\logit\lambda$ &    $\logit\phi$ &  $\tau$ &        Total \\
\midrule
Full  &   0.29 (0.59) &   0.04 (0.03) &   0.18 (0.34) &  0.55 (0.11) &  0.97 (0.05) \\
No Buffer &   0.07 (0.08) &  -0.38 (0.09) &  -0.15 (0.10) &  0.56 (0.10) &  0.77 (0.08) \\
Buffer &  -0.27 (0.24) &  -0.72 (0.19) &  -0.69 (0.17) &  0.12 (0.19) &  \textbf{0.39 (0.09)} \\
\bottomrule
\end{tabular}
\end{sc}
\end{small}
\vskip 0.2in

\caption{KSD results for SVM on exchange rate data.}
\label{tab:svm-exchange}
\vskip 0.05in
\begin{small}
\begin{sc}
\begin{tabular}{l|llll}
\toprule
  & \multicolumn{4}{c}{$\log_{10}$KSD} \\
\midrule
Grad Est. & $\phi$ & $\sigma$ &  $\tau$ & Total \\
\midrule
Full &  3.63 (0.30) &  3.76 (0.07) &  1.46 (0.38) &  4.03 (0.14) \\
Weekly &  3.86 (0.08) &  2.18 (0.28) &  0.67 (0.39) &  3.87 (0.08) \\
No Buffer &  4.48 (0.01) &  1.84 (0.15) &  1.21 (0.14) &  4.48 (0.01) \\
Buffer  &  3.53 (0.11) &  2.32 (0.13) &  1.23 (0.05) &  \textbf{3.56 (0.10)} \\
\bottomrule
\end{tabular}
\end{sc}
\end{small}
\vskip 0.2in

\caption{KSD results for GARCH on exchange rate data.}
\label{tab:garch-exchange}
\vskip 0.05in
\begin{small}
\begin{sc}
\begin{tabular}{l|lllll}
\toprule
  & \multicolumn{4}{c}{$\log_{10}$KSD} \\
\midrule
Grad Est. & $\log\mu$ & $\logit\lambda$ &    $\logit\phi$ &  $\tau$ &        Total \\
\midrule
Full &  2.18 (0.67) &   2.18 (0.07) &  2.19 (0.61) &  2.07 (0.06) &  2.84 (0.30) \\
Weekly  &  2.17 (0.51) &   2.21 (0.03) &  2.31 (0.29) &  1.85 (0.19) &  2.81 (0.21) \\
No Buffer &  1.76 (0.06) &   1.43 (0.46) &  1.31 (0.09) &  1.58 (0.08) &  \textbf{2.09 (0.09)} \\
Buffer &  1.76 (0.03) &   2.01 (0.08) &  1.11 (0.07) &  1.87 (0.07) & \textbf{2.19 (0.05)} \\
\bottomrule
\end{tabular}
\end{sc}
\end{small}
\end{center}
\end{table*}

\end{document}